\documentclass{article}

\usepackage{microtype}
\usepackage{graphicx}
\usepackage{subfigure}
\usepackage{booktabs} % for professional tables

\usepackage{hyperref}

\usepackage[accepted]{icml2018}

\usepackage{times}
\usepackage{helvet}
\usepackage{courier}
\usepackage{subfigure} 
\usepackage{cancel}
\RequirePackage{amsthm,amsmath}
\usepackage{graphicx, amssymb, mathrsfs,amsfonts,url, bm}

\newtheorem{lemma}{Lemma}[section]

\newcommand{\E}{\mathrm{E}}

\newcommand{\bbeta}{\boldsymbol\beta}

\newcommand{\bz}{\bm{z}}
\newcommand{\mathcalX}{\mathcal{X}}
\newcommand{\bpi}{\bm{\pi}}

\newcommand{\bx}{\mathbf{x}}

\newcommand{\bxi}{\mathbf{x}_i}
\newcommand{\xknoi}{\mathcal{X}_{k,-i}}

\newcommand{\bznoi}{\mathbf{z}_{-i}}

\newcommand{\niw}{\mathrm{NIW}}

\newcommand{\bmu}{\boldsymbol\mu}
\newcommand{\bsigma}{\boldsymbol\Sigma}

\newcommand{\multinomial}{\mathrm{Multinomial}}
\newcommand{\dirichlet}{\mathrm{Dirichlet}}

\icmltitlerunning{Reducing over-clustering via the powered Chinese restaurant process}

\begin{document}

\twocolumn[
\icmltitle{Reducing over-clustering via the powered Chinese restaurant process}




\begin{icmlauthorlist}
\icmlauthor{Jun Lu}{}
\icmlauthor{Meng Li}{rice}
\icmlauthor{David Dunson}{duke}
\end{icmlauthorlist}

\icmlaffiliation{rice}{Department of Statistics, Rice University, Houston, TX, USA}
\icmlaffiliation{duke}{Department of Statistical Science, Duke University, Durham, NC, USA}

\icmlcorrespondingauthor{Jun Lu}{jun.lu.locky@gmail.com}
%\icmlcorrespondingauthor{Meng Li}{meng@rice.edu}
%\icmlcorrespondingauthor{David Dunson}{dunson@duke.edu}

\icmlkeywords{Machine Learning, ICML}

\vskip 0.3in
]



\printAffiliationsAndNotice{}  % leave blank if no need to mention equal contribution

\begin{abstract}
Dirichlet process mixture (DPM) models tend to produce many small clusters regardless of whether they are needed to accurately characterize the data - this is particularly true for large data sets.
However, interpretability, parsimony, data storage and communication costs all are hampered by having overly many clusters.  We propose a powered Chinese restaurant process to limit this kind of problem and penalize over clustering. The method is illustrated using some simulation examples and data with large and small sample size including MNIST and the Old Faithful Geyser data.
\end{abstract}
\section{Introduction}
\noindent Dirichlet process mixture (DPM) models and closely related formulations have been very widely used for flexible modeling of data and for clustering.  DPMs of Gaussians have been shown to possess frequentist optimality properties in density estimation, obtaining minimax adaptive rates of posterior concentration with respect to the true unknown smoothness of the density \cite{shen2011adaptive}.  DPMs are also very widely used for probabilistic clustering of data.  In the clustering context, it is well known the DPMs favor introducing new components at a log rate as the sample size increases, and tend to produce some large clusters along with many small clusters.  As the sample size $N$ increases, these small clusters can be introduced as an artifact even if they are not needed to characterize the true data generating process; for example, even if the true model has finitely many clusters, the DPM will continue to introduce new clusters as $N$ increases \cite{miller2013simple,argiento2009comparison,lartillot2004bayesian,
onogi2011characterization,miller2014inconsistency}.

Continuing to introduce new clusters as $N$ increases can be argued to be an appealing property. The number of `types' of individuals is unlikely to be finite in an infinitely large population, and there is always a chance of discovering new types as new samples are collected.  This rationale has motivated a rich literature on generalizations of Dirichlet processes, which have more flexibility in terms of the rate of introduction of new clusters.  For example, the two parameter Poisson-Dirichlet process (a.k.a., the Pitman-Yor process) is a generalization that instead induces a power law rate, which is more consistent with many observed data processes  \cite{perman1992size}. There has also been consideration of a rich class of Gibbs-type processes, which considerably generalize Pitman-Yor to a broad class of so-called exchangeable partition probability functions (EPPFs) \cite{gnedin2006exchangeable,lijoi2010models,de2015gibbs,
bacallado2017sufficientness,favaro2013conditional}.  Much of the emphasis in the Gibbs-type process literature has been on data in which `species' are observed directly, and the goal is predicting the number of new species in a further sample \cite{lijoi2007bayesian}. It remains unclear whether such elaborate generalizations of Dirichlet processes have desirable behavior when clusters/species are latent variables in a mixture model.

The emphasis of this article is on addressing practical problems that arise in implementing DPMs and generalizations when sample sizes and data dimensionality are moderate too large.  In such settings, it is common knowledge that the number of clusters can be too large, leading to a lack of interpretability, computational problems and other issues.  For these reasons, it is well motivated to develop {\em sparser} clustering methods that do not restrict the number of clusters to be finite {\em a priori} but instead favor deletion of small clusters that may not be needed to accurately characterize the true data generating mechanism.  With this goal in mind, we find that the usual focus on exchangeable models, and in particular EPPFs, can limit practical performance.  There has been some previous work on non-exchangeable clustering methods motivated by incorporation of predictor-dependence in clustering \cite{blei2011distance,ghosh2011spatial,soumya2014nonparametric,socher2011spectral}, but our focus is instead on providing a simple approach that tends to delete small and unnecessary clusters produced by a DPM.  Marginalizing out the random measure in the DPM specification produces a Chinese restaurant process (CRP).  We propose a simple powered modification to the CRP, which has the desired impact on clustering and develop associated inference methods.

\section{Powered Chinese restaurant process (pCRP)}
\subsection{The Chinese restaurant process (CRP)}
The Chinese restaurant process is a simple stochastic process that is exchangeable. In the analogy from which this process takes its name, customers seat themselves at a restaurant with an infinite number of tables. Each customer sits at a previously occupied table with probability proportional to the number of customers already sitting there, and at a new table with probability proportional to a concentration parameter $\alpha$. For example, the first customer enters and sits at the first table. The second customer enters and sits at the first table with probability $\frac{1}{1+\alpha}$ and at a new table with probability $\frac{\alpha}{1+\alpha}$. The $i^{th}$ customer sits at an occupied table with probability proportional to the number of customers already seated at that table, or sits at a new table with a probability proportional to $\alpha$. Formally, if $z_i$ is the table chosen by the $i^{th}$ customer, then
\begin{equation}
\begin{aligned}
&\hspace*{-0.04in} p(z_i = k | \bznoi, \alpha)=\\
& \hspace*{-0.08in}  \left\{
                \begin{array}{ll}
                  \frac{N_{k,-i}}{N+\alpha-1},  \text{ if \textit{k} is occupied, i.e. }  N_k > 0, \\
                  \frac{\alpha}{N+\alpha-1}, \text{ if \textit{k} is a new table, i.e. }  k = k^{\star} = K + 1, \\
                \end{array}
              \right.
\label{equation:crp_equation}
\end{aligned}
\end{equation}
where $\bznoi = (z_1, z_2, \ldots , z_{i-1}, z_{i+1},  \ldots, z_N)$ and $N_{k,-i}$ is the number of customers seated at table $k$ excluding customer $i$.
From the definition above, we can observe that the CRP is defined by a rich-get-richer property in which the probability of being allocated to a table increases in proportion to the number of customers already at that table.

In a CRP mixture model, each table is assigned a specific parameter in a kernel generating data at the observation level.  Customers assigned to a specific table are given the cluster index corresponding to that table, and have their data generated from the kernel with appropriate cluster/table-specific parameters.  The CRP provides a prior probability model on the clustering process, and this prior can be updated with the observed data to obtain a posterior over the cluster allocations for each observation in a data set.  The CRP provides an exchangeable prior on the partition of indices $\{1,\ldots,N\}$ into clusters; exchangeability means that the ordering of the indices has no impact on the probability of a particular configuration -- only the number of clusters $K_N$ and the size of each cluster can play a role.  The CRP implies that $\E[K_N|\alpha] = O(\alpha \log N)$~\cite{teh2011dirichlet}.

\subsection{Powered Chinese restaurant process}

Popular Bayesian nonparametric priors, such as the Dirichlet process \cite{ferguson1973bayesian,blackwell1973ferguson,antoniak1974mixtures}, Chinese restaurant process, Pitman-Yor process \cite{perman1992size,pitman1997two} and Indian buffet process \cite{griffiths2005infinite,thibaux2007hierarchical}, assume infinite exchangeability.  In particular, suppose we have a clustering process for an infinite sequence of data points $i=1,2,3,\ldots,\infty$.  This clustering process will induce a partition of the integers $\{1,\ldots,N\}$ into $K_N$ clusters of size $N_1,\ldots,N_{K_N}$, for $N=1,2,\ldots,\infty$. For an exchangeable clustering process, the probability of a particular partition of $\{1,\ldots,N\}$ only depends on $N_1,\ldots,N_{K_N}$ and $K_N$, and does not depend on the order of the indices $\{1,\ldots,N\}$.  In addition, the probability distributions for different choices of $N$ are {\em coherent};  the probability distribution of partitions of $\{1,\ldots,N\}$ can be obtained from the probability distribution of partitions of $\{1,\ldots,N+1\}$ by marginalizing out the cluster assignment for data point $i=N+1$.  These properties are often highly appealing computationally and theoretically, but it is nonetheless useful to consider processes that violate the infinite exchangeability assumption.  This can occur when the addition of a new data point $i=N+1$ to a sample of $N$ data points can impact the clustering of the original $N$ data points.  For example, we may re-evaluate whether data point $1$ and $2$ are clustered together in light of new information provided by a third data point, a type of {\em feedback} property.  

We propose a new powered Chinese restaurant process (pCRP), which is designed to favor elimination of artifactual small clusters produced by the usual CRP by implicit incorporation of a feedback property violating the usual exchangeability assumptions.  The proposed pCRP makes the random seating assignment of the customers depend on the powered number of customers at each table (i.e. raise the number of each table to power $r$). Formally, we have
\begin{equation}
\begin{aligned}
&p(z_i = k | \bznoi, \alpha)=\\
&\left\{
                \begin{array}{ll}
                  \frac{N_{k,-i}^r}{\sum_h^K N_{h,-i}^r+\alpha}, \text{ if \textit{k} is occupied, i.e. }  N_k > 0,\\
                  \frac{\alpha}{\sum_h^K N_{h,-i}^r+\alpha}, \text{ if \textit{k} is a new table, i.e. }  k = k^{\star}=K+1,
                \end{array}
              \right.
\label{equation:powered_crp_equation}
\end{aligned}
\end{equation}
where $r>1$ and $N_{k,-i}$ is the number of customers seated at table $k$ excluding customer $i$. More generally, one may consider a $g$-CRP to generalize the CRP such that 
\begin{equation}
\begin{aligned}
&p(z_i = k | \bznoi, \alpha)=\\
&\left\{
\begin{array}{ll}
\frac{g(N_{k,-i})}{\sum_h^K g(N_{h,-i})+\alpha}, \text{ if \textit{k} is occupied, i.e. }  N_k > 0,\\
\frac{\alpha}{\sum_h^K g(N_{h,-i})+\alpha}, \text{ if \textit{k} is a new table, i.e. }  k = k^{\star}=K+1,
\end{array}
\right.
\label{equation:g_crp_equation}
\end{aligned}
\end{equation}
where $g(\cdot): \mathbb{R^+} \rightarrow \mathbb{R^+} $ is an increasing function and $g(0) = 0$. We achieve shrinkage of small clusters via a rich-get-(more)-richer property by requiring $g(x) \geq x$ for $x > 1$ to`enlarge'  clusters containing more than one element. We require the $g$-CRP to maintain a \textit{proportional invariance} property:
\begin{equation}\label{eq:proportional.invariant}
\frac{g(cN_1)}{g(cN_2)} = \frac{g(N_1)}{g(N_2)}
\end{equation}
for any $c, N_1, N_2 > 0$, so that scaling cluster sizes by a constant factor has no impact on the prediction rule in \eqref{equation:g_crp_equation}.  The following Lemma~\ref{lemma:unique} shows that the pCRP in equation~\eqref{equation:powered_crp_equation} using the power function is the only $g$-CRP that satisfies the proportional invariance property. 
\begin{lemma}
	\label{lemma:unique}
	If a continuous function $g(x): \mathbb{R^+} \rightarrow \mathbb{R^+}$ satisfies equation~\eqref{eq:proportional.invariant}, then $g(x) = g(1) \cdot x^r$ for all $x > 0$ and some constant $r \in \mathbb{R}$. 
\end{lemma} 
\begin{proof}[Proof of Lemma~\ref{lemma:unique}]
	It is easy to verify that  $g(x) = g(1) \cdot x^r$ for some $r > 0$ is a solution to the functional equation~\eqref{eq:proportional.invariant}. We next show its uniqueness.
	
	Equation~\eqref{eq:proportional.invariant} implies that $g(cN_1)/g(N_1) = g(cN_2)/g(N_2)$ for any $N_1, N_2 > 0$. Denote $f(c) = g(cN)/g(N) > 0$ for arbitrary $N > 0$. We then have $f(s t) = g(stN)/g(N) = g(stN)/g(tN) \cdot g(tN) / g(N) = f(s) f(t)$ for any $s, t > 0$. By letting $f^*(x) = f(e^x) > 0$, it follows that $\log f^*(s + t) = \log f^*(s) + \log f^*(t)$, which is the well known Cauchy functional equation and has the unique solution $\log f^*(x) = r x$ for some constant $r$. Therefore, $f(x) = f^*(\log(x)) = x^r$ which gives $g(cN) = g(N) c^r$. We complete the proof by letting $N = 1$. 
\end{proof}

%DELETE: pCRP is not \textit{marginal invariant}, meaning that if we integrate out one observation we would not get the same probability distribution as if that observation was not part of the model. 
As a generalization of the CRP, which corresponds to the special case in which $r=1$, the proposed pCRP with $r > 1$ generates new clusters following a probability that is configuration dependent and not exchangeable. For example, for three customers $z_1, z_2, z_3$, $p(z_3 = 2 \mid z_1=1, z_2=1) < p( z_3 = 1 \mid z_1=1, z_2=2)$, where $z_i = k$ if the $i^{th}$ customer sits at table $k$. 
%DELETE: Thus we have following properties in pCRP: (a). Taking the order of customers into account, pCRP favors to make customers spread first, then to distribute following customers to old tables.
This non-exchangeability is a critical feature of pCRP, allowing new cluster generation to learn from existing patterns. Consider two extreme configurations: (i) $K_N = N$ with one member in each cluster, and (ii) $K_N = 1$ with all members in a single cluster. The probabilities of generating a new cluster under (i) and (ii) are both $\alpha/(N + \alpha)$ in CRP, but dramatically different in pCRP: (i) $\alpha/(N + \alpha)$ and (ii) $\alpha/(N^r + \alpha)$, respectively. Therefore, if the previous customers are more spread out, there is a larger probability of continuing this pattern by creating new tables.  Similarly, if customers choose a small
number of tables, then a new customer is more likely to join the dominant clusters rather than open a new table.

The power $r$ is a critical parameter controling how much we penalize small clusters. The larger the power $r$, the greater the penalty.  We propose a method to choose $r$ in a data-driven fashion: cross validation using a proper loss function to select a fixed $r$.

\subsection{Power parameter tuning}\label{sec:pcrp_parameter_calibration}
The proportional invariance property makes it easier to define a cross validation (CV) procedure for estimating $r$.  In particular, one can tune $r$ to obtain good performance on an initial training sample and that $r$ would also be appropriate for a subsequent data set that has a very different sample size.  For other choices of $g(\cdot)$, which do not possess proportional invariance, it may be necessary to adapt $r$ to the sample size for appropriate calibration.   

In evaluating generalization error, we use the following loss function based on within-cluster sum of squares:
\begin{equation}
 \sum_{k=1}^K  \sqrt{ \sum_{j: j \in C_k}^{N_k} ||\bx_j - \overline{\bx}_k||^2 },
\end{equation}
where $C_k$ is the data samples in the $k$th cluster and $ \overline{\bx}_k$ is the mean vector for cluster $k$. The square root has an important impact in favoring a smaller nunber of clusters; for example, inducing a price to be paid for introducing two clusters with the same mean.  In implementing CV, we start by choosing a small value of $r$ ($r=1+\epsilon$) and then increasing until we identify an inflection point.

\begin{figure}[h!]
\centering
  \includegraphics[width=0.3\textwidth]{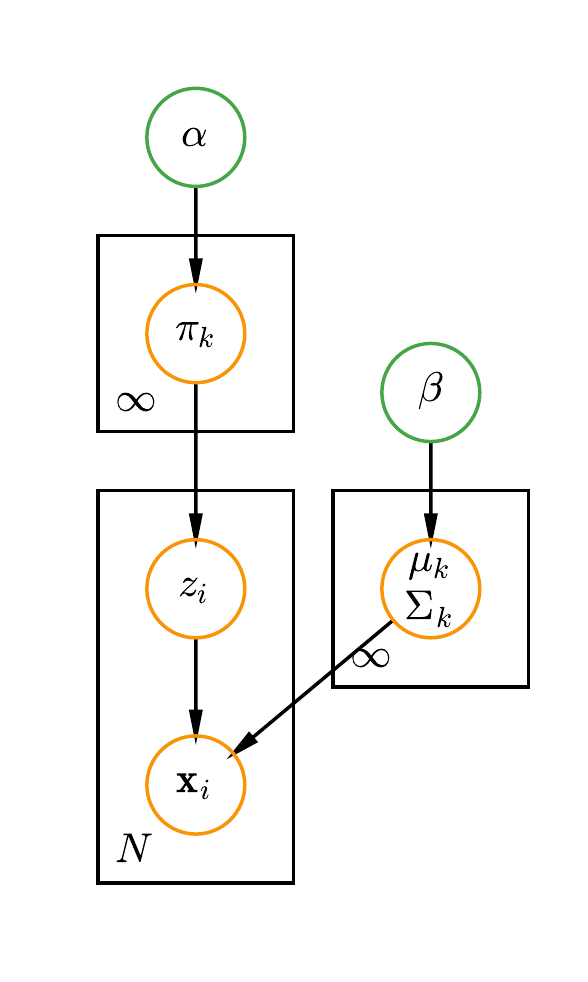}
  \caption{A Bayesian infinite GMM.}
  \label{fig:gmm_infinite_without_hyper}
\end{figure}

\subsection{Posterior inference by collapsed Gibbs sampling}
Although the proposed pCRP is generic, we focus on its application in Gaussian mixture models (GMMs) for concreteness. We develop a collapsed Gibbs sampling algorithm (Alg 3 in \cite{neal2000markov} and further introduced in \cite{murphy2012machine}) for posterior computation. Our proposed pCRP can also be easily implemented via a non-collapsed Gibbs sampling algorithm (\cite{west1993hierarchical}, i.e. Alg 2 in \cite{neal2000markov}). In addition, we permute the data at each sampling iteration to eliminate order dependence as in~\cite{socher2011spectral}.

Let $\mathcalX$ be the observations, assumed to follow a mixture of multivariate Gaussian distributions. We use a conjugate Normal-Inverse-Wishart (NIW) prior $p(\bmu, \bsigma | \bbeta)$ for the mean vector $\bmu$ and covariance matrix $\bsigma$ in each multivariate Gaussian component, where $\bbeta$ consists of all the hyperparameters in NIW. I.e. we will work with the following definition of Bayesian infinite Gaussian mixture model:
\begin{equation}
\begin{aligned}
\bx_i | z_i, \{\bmu_k, \bsigma_k \} &\sim \mathcal{N}(\bmu_{z_i}, \bsigma_{z_i}), \\
z_i | \bpi    									&\sim \multinomial(\pi_1, \ldots, \pi_K), \\
\{\bmu_k, \bsigma_k\}					&\sim \niw(\bbeta), \\
\bpi 											&\sim \dirichlet(\alpha/K, \ldots, \alpha/K),
\end{aligned}
\end{equation}
by taking the limit as $K \rightarrow \infty$ \cite{rasmussen1999infinite}, where $\mathcal{N}(\cdot, \cdot)$ is the multivariate normal distribution, $\multinomial(\cdot)$ is the Multinomial distribution and $\dirichlet(\cdot)$ is the Dirichlet distribution. The process is illustrated in Figure~\ref{fig:gmm_infinite_without_hyper}.

A key quantity in a collapsed Gibbs sampler is the probability of each customer $i$ sitting with table $k$: $p(z_i = k | \bznoi, \mathcalX, \alpha, \bbeta)$, where $\bznoi$ are the seating assignments of all the other customers and $\alpha$ is the concentration parameter in CRP and pCRP. This probability is calculated as follows:
\begin{equation} 
\begin{aligned}
&p(z_i = k| \bznoi, \mathcal{X} , \alpha, \bbeta)  \\
& \varpropto p(z_i = k | \bznoi, \alpha, \cancel{\bbeta})  p(\mathcal{X} |z_i = k, \bznoi, \cancel{\alpha}, \bbeta) \\
& = p(z_i = k| \bznoi, \alpha) p(\bxi |\mathcal{X}_{-i}, z_i = k, \bznoi, \bbeta) 
\\
&\,\,\, \cdot p(\mathcal{X}_{-i} |\cancel{z_i = k}, \bznoi, \bbeta)\\
& \varpropto p(z_i = k| \bznoi, \alpha) p(\bxi|\mathcal{X}_{-i}, z_i = k, \bznoi, \bbeta) \\
& \varpropto p(z_i = k| \bznoi, \alpha) p(\bxi | \xknoi, \bbeta),\\
\end{aligned}
\label{equation:pcrp_ifmm_collabsed_gibbs}
\end{equation}  
where $\xknoi$ are the observations in table $k$ excluding the $i^{th}$ observation and the first term of the last equation above $p(z_i = k| \bznoi, \alpha)$ is the proposed pCRP in equation~\eqref{equation:powered_crp_equation}. 
If $z_i = k$ is an existing component, the second term above $p(\bxi | \xknoi, \bbeta)$ is calculated using the posterior predictive distribution at $\bxi$, where the posterior prediction distribution of the new data $\bx^\star$ given the data set $\mathcalX$ and the prior parameter $\bbeta$ under the NIW prior is
\begin{equation}
p(\bx^\star | \mathcalX, \bbeta) = \int_{\bmu}  \int_{\bsigma} p(\bx^\star | \bmu, \bsigma) p(\bmu,\bsigma | \mathcalX, \bbeta) d\bmu d \bsigma.
\end{equation}
%Alternatively we could calculate the numerator and denominator separately by the marginal likelihood of data under NIW prior:
%\begin{equation}
%\begin{aligned}
%p(\mathcal{X}|\bbeta) &= \int_{\bmu} \int_{\bsigma} p(\mathcal{X}, \bmu, \bsigma | \bbeta) d\bmu d\bsigma \\
%	&= \int_{\bmu} \int_{\bsigma} p(\mathcal{X}|\bmu, \bsigma) p(\bmu, \bsigma|\bbeta)  d\bmu d\bsigma.
%\end{aligned}
%\end{equation}
When $z_i = k^\star$ is a new component then we have:
\begin{equation}
\begin{aligned}
&p(\bxi | \xknoi, \bbeta) = p(\bxi | \bbeta) \\
&= \int_{\bmu}  \int_{\bsigma} p(\bxi | \bmu, \bsigma) p(\bmu, \bsigma|\bbeta) d\bmu d \bsigma,
\end{aligned}
\end{equation}
which is just the prior predictive distribution and can be calculated by the posterior predictive distribution for new data $p(\bx^\star | \mathcalX, \bbeta)$ under NIW prior but with $\mathcalX = \{\emptyset\}$.

Algorithm~\ref{algo:pcrp_ifmm_plain_gibbs} gives the pseudo code of the collapsed Gibbs sampler to implement pCRP in Gaussian mixture models.

\begin{algorithm}[tb]
   \caption{Collapsed Gibbs Sampling for pCRP}
   \label{alg:example}
\begin{algorithmic}
\STATE {\bfseries Input:} Choose an initial $\bz$, $r$, $\alpha$, $\bbeta$;
\FOR{$T$ iterations}   
\STATE $\bullet$ Sample random permutation $\tau$ of $1, \ldots, N$;
	\FOR{$i \in (\tau(1), \ldots, \tau(N))$}
		\STATE $\bullet$ Remove $\bxi$'s statistics from component $z_i$;
		\FOR{$k=1$ {\bfseries to} $K$}
		\STATE $\bullet$ Calculate $p(z_i=k| \bznoi, \alpha) =  \frac{N_{k,-i}^r}{\sum_h^K N_{h,-i}^r+\alpha}$;
		\STATE $\bullet$ Calculate $p(\bxi | \xknoi, \bbeta)$;
		\STATE $\bullet$ Calculate $p(z_i = k | \bznoi, \mathcal{X}, \alpha, \bbeta) \propto p(z_i=k| \bznoi, \alpha) p(\bxi | \xknoi, \bbeta)$;
		\ENDFOR
		\STATE $\bullet$ Calculate $p(z_i = k^\star | \bznoi, \alpha)= \frac{\alpha}{\sum_h^K N_{h,-i}^r+\alpha}$;
		\STATE $\bullet$ Calculate $p(\bxi | \bbeta)$;
		\STATE $\bullet$ Calculate $p(z_i = k^\star | \bznoi, \mathcalX, \alpha, \bbeta) \propto p(z_i = k^\star | \bznoi, \alpha) p(\bxi | \bbeta)$;		
		\STATE $\bullet$ Sample $k_{new}$ from $p(z_i | \bznoi, \mathcalX, \alpha, \bbeta)$ after normalizing;
		\STATE $\bullet$ Add $\bxi$'s statistics to the component $z_i=k_{new}$ ;
	     \STATE $\bullet$ If any component is empty, remove it and decrease $K$.
	\ENDFOR
\ENDFOR
\end{algorithmic}\label{algo:pcrp_ifmm_plain_gibbs}
\end{algorithm}

\section{Experiments}
We conduct experiments to demonstrate the main advantages of the proposed pCRP using both synthetic and real data. In a wide range of scenarios across various sample sizes, pCRP reduces over-clustering of CRP, and leads to performances that are as good or better than CRP in terms of density estimation, out of sample prediction, and overall clustering results. 

In all experiments, we run the Gibbs sampler 20, 000 iterations with a burn-in of 10, 000. The sampler is thinned by keeping every 5$^{th}$ draw. We use the same concentration parameter $\alpha = 1$ for both CRP and pCRP in all scenarios. In addition, we equip CRP with an unfair advantage to match the magnitude of its prior mean $\alpha \log(N)$ to the true number of clusters, termed {\it CRP-Oracle}. The power $r$ in pCRP is tuned using cross validation. In order to measure overall clustering performance, we use normalized mutual information (NMI) \cite{mcdaid2011normalized} and variation of information (VI) \cite{meilua2003comparing}, which measures the similarity between the true and estimated cluster assignments. Higher NMI and lower VI indicate better performance.  If applicable, metrics using the true clustering are calculated to provide an upper bound for all methods, coded as `Ground Truth'. 

\subsection{Simulation experiments}\label{sec:simulated_examples} 
We first use simulated data to assess the performance of pCRP in emptying extra components, compared to the traditional CRP.  Figure~\ref{fig:pcrp_prior_densities} shows the true data generating density, which represent the two cases of well-mixed Gaussian components and shared mean Gaussian mixture coded as Sim 1 and Sim 2, respectively. 

The oracle concentration parameters in CRP-Oracle are (0.52, 0.40) in Sim 1 and (0.52, 0.40) in Sim 2, which are all smaller than the unit concentration parameter used in CRP and pCRP.  The sample sizes in the two simulation cases are respectively (300, 2000).  Figure~\ref{fig:pcrp_cross-validation} shows the cross validation curve to select $r$ in pCRP using a training data set with 200 samples. The representative cases of inflection point described in Section \ref{sec:pcrp_parameter_calibration} were observed: the loss curve for cross validation `blows up' for one particular $r$ value in both Sim 1 and Sim 2.  We can find the first stage of Sim 2 is oscillating from 14.2595 to 14.2005 which is flatter than that of Sim 1 that oscillates from 15.5260 to 14.9020. This is because the components in Sim 2 are well separated. We choose this change point as the power $r$ in either case. 

\begin{figure}[h!]
\center
\subfigure[Sim 1]{\includegraphics[width=0.228\textwidth]{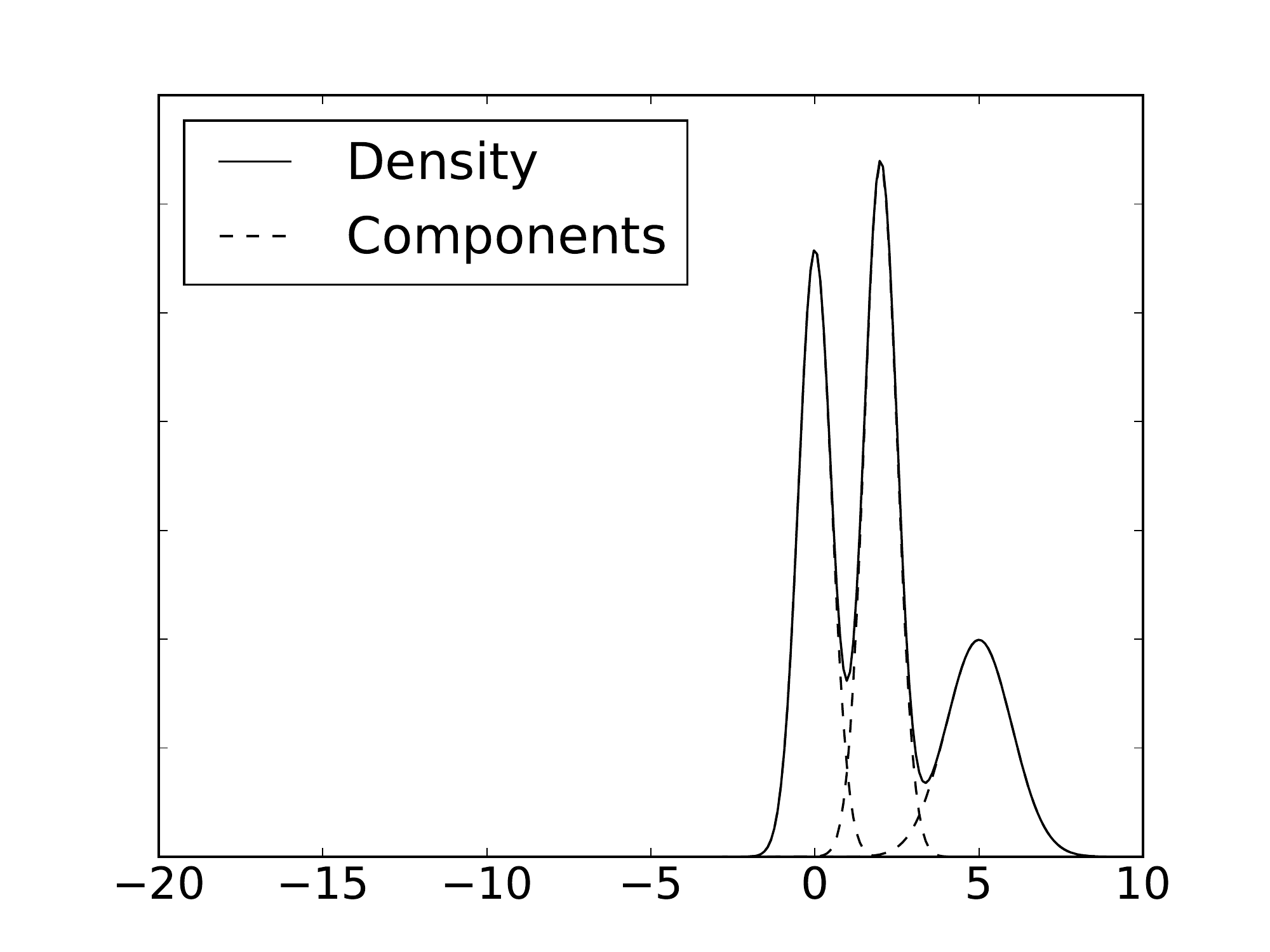} \label{fig:true_density_sim2}}
~
\subfigure[Sim 2]{\includegraphics[width=0.228\textwidth]{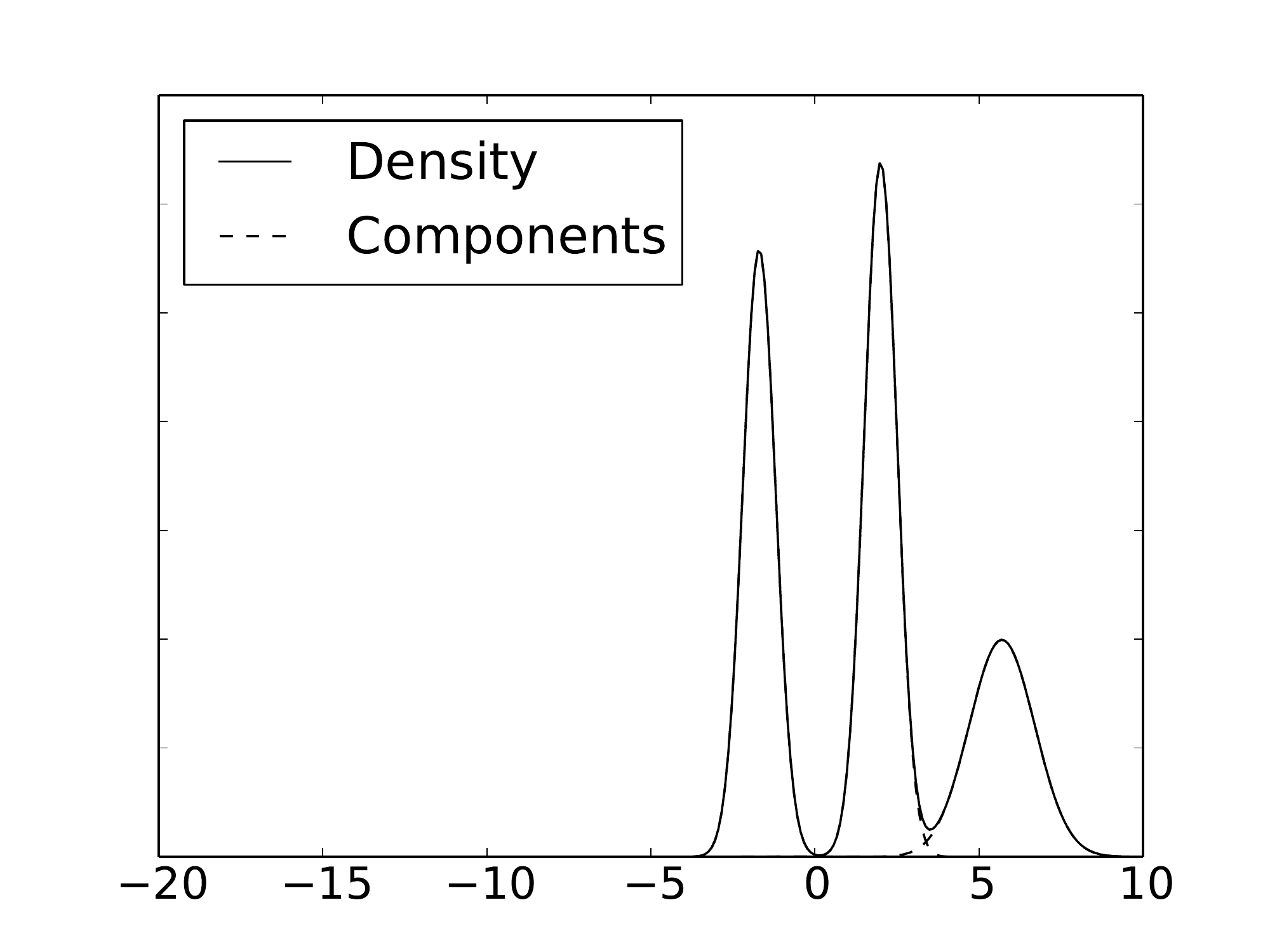} \label{fig:true_density_sim6}}
\center
\caption{Data generating densities for two scenarios: (a) Sim1: a mixture of three poorly separated Gaussian components; (b) Sim 2: a mixture of three well separated Gaussian components.}
\label{fig:pcrp_prior_densities}
\end{figure}

\begin{figure}[h!]
\center
\subfigure[Sim 1]{\includegraphics[width=0.228\textwidth]{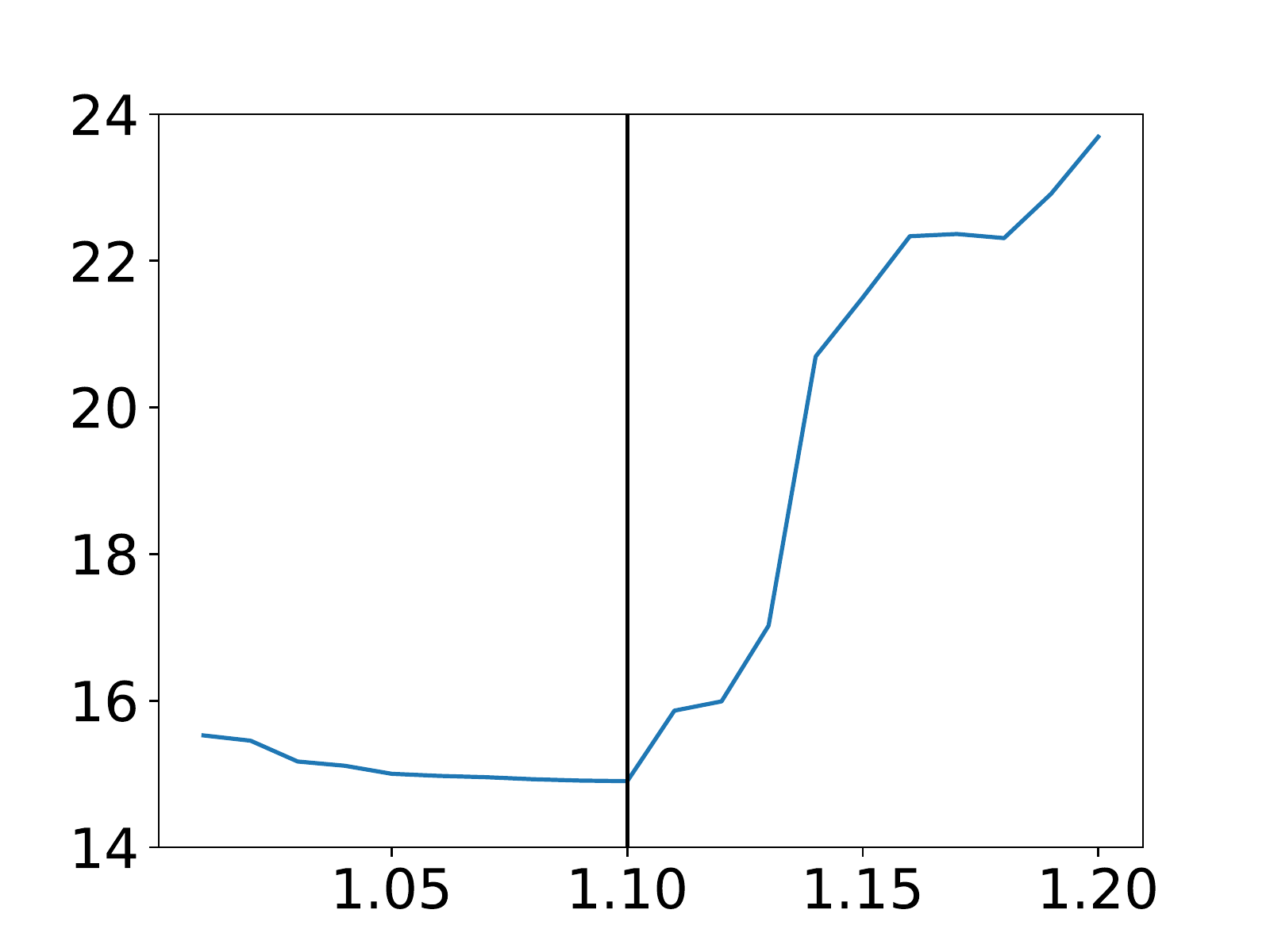} \label{fig:pcrp_cv_sim2}}
~
\subfigure[Sim 2]{\includegraphics[width=0.228\textwidth]{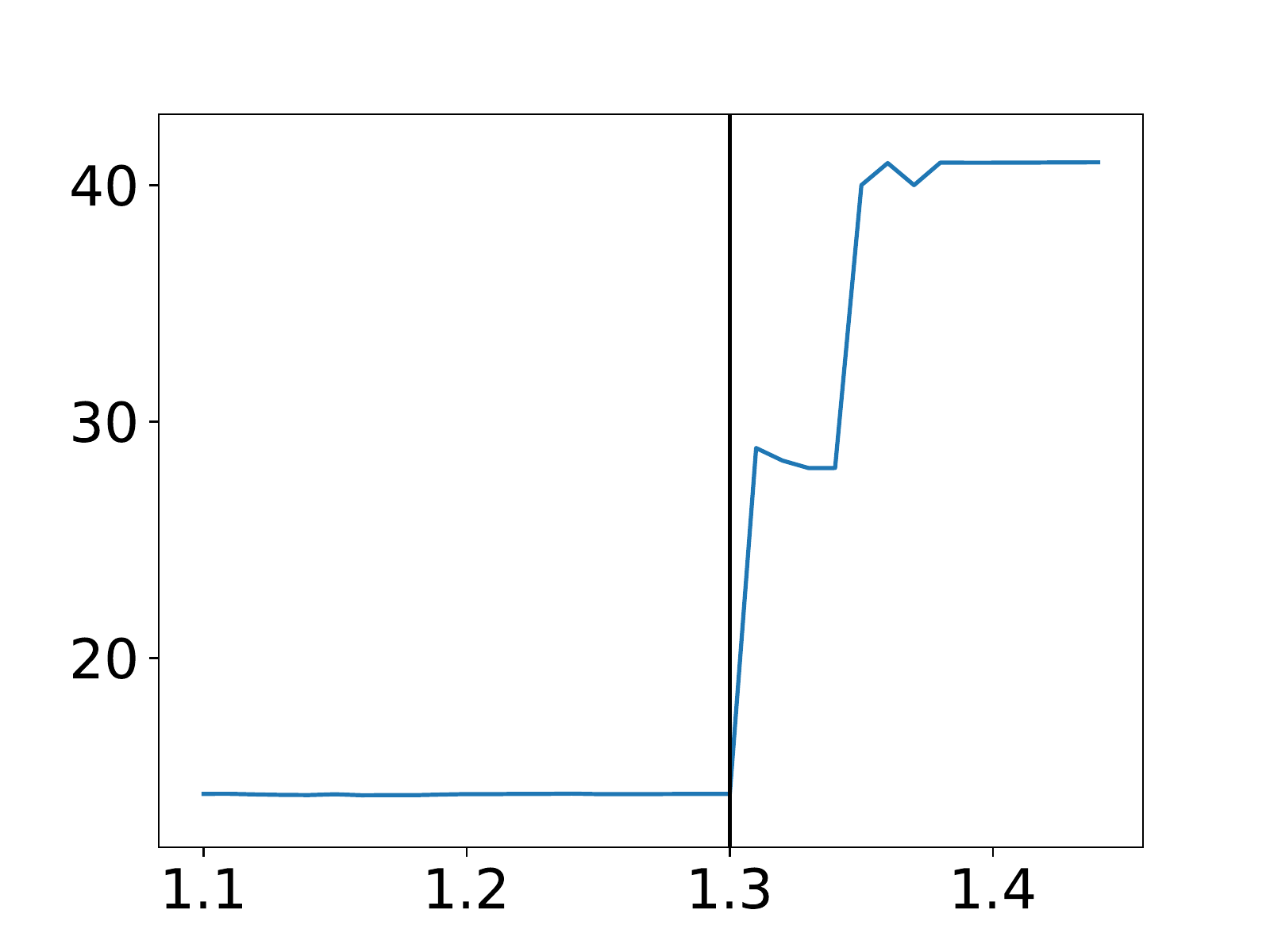} \label{fig:pcrp_cv_sim6}}
\center
\caption{Cross validation curves to choose $r$ for Sim 1 and Sim 2. The $x$-axis is the power value, the $y$-axis is the loss. The vertical line is the chosen power $r$ value.}
\label{fig:pcrp_cross-validation}
\end{figure}

%%%%%%%%% traceplot for cluster number
\begin{figure}[!h]
	\center
	\subfigure[ CRP-Oracle]{\includegraphics[width=0.40\textwidth]{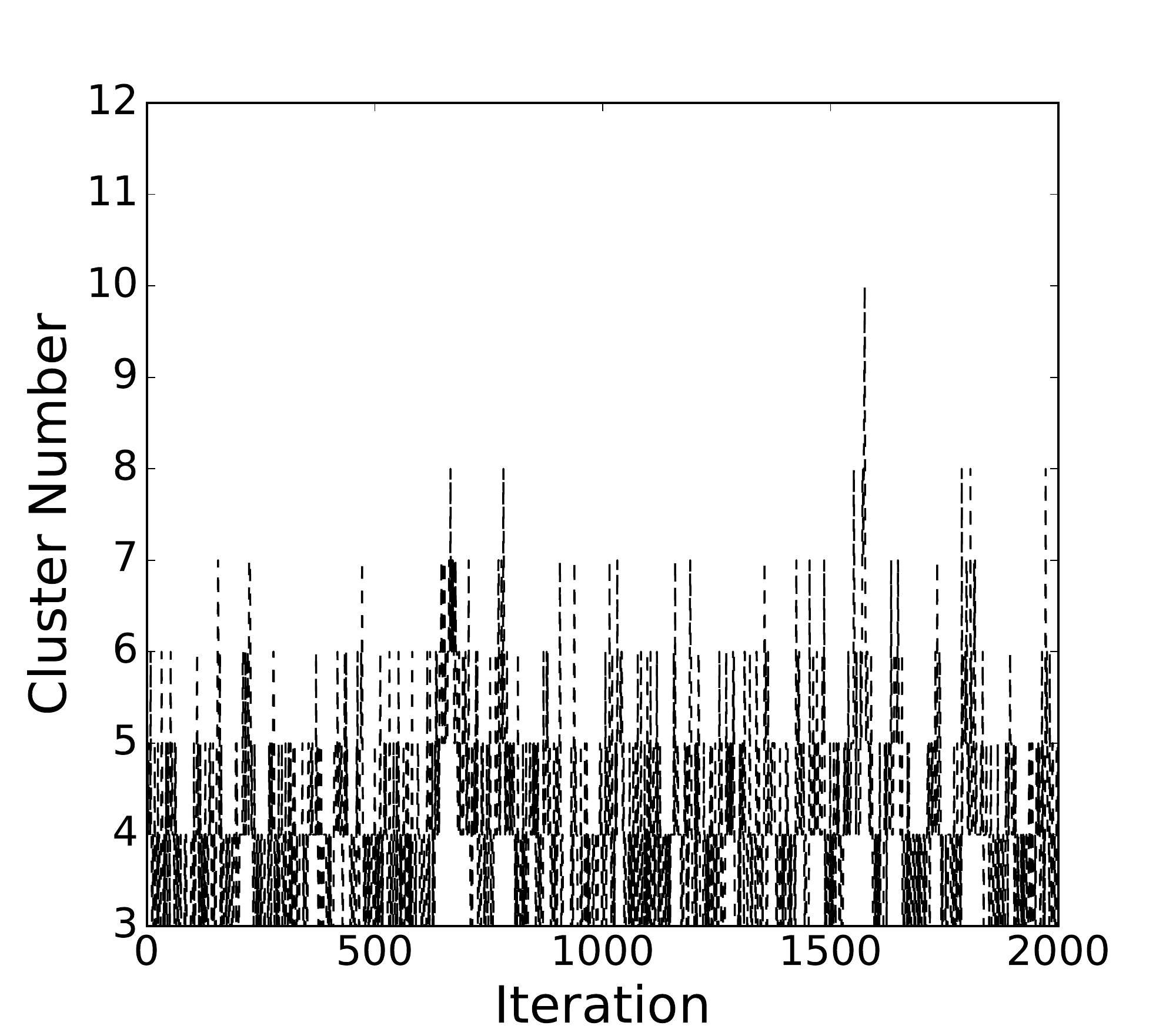} \label{fig:pcrp_clusternum_3methods_traceplot_sim1_crp04}}
	\hfill
	\subfigure[ CRP]{\includegraphics[width=0.23\textwidth]{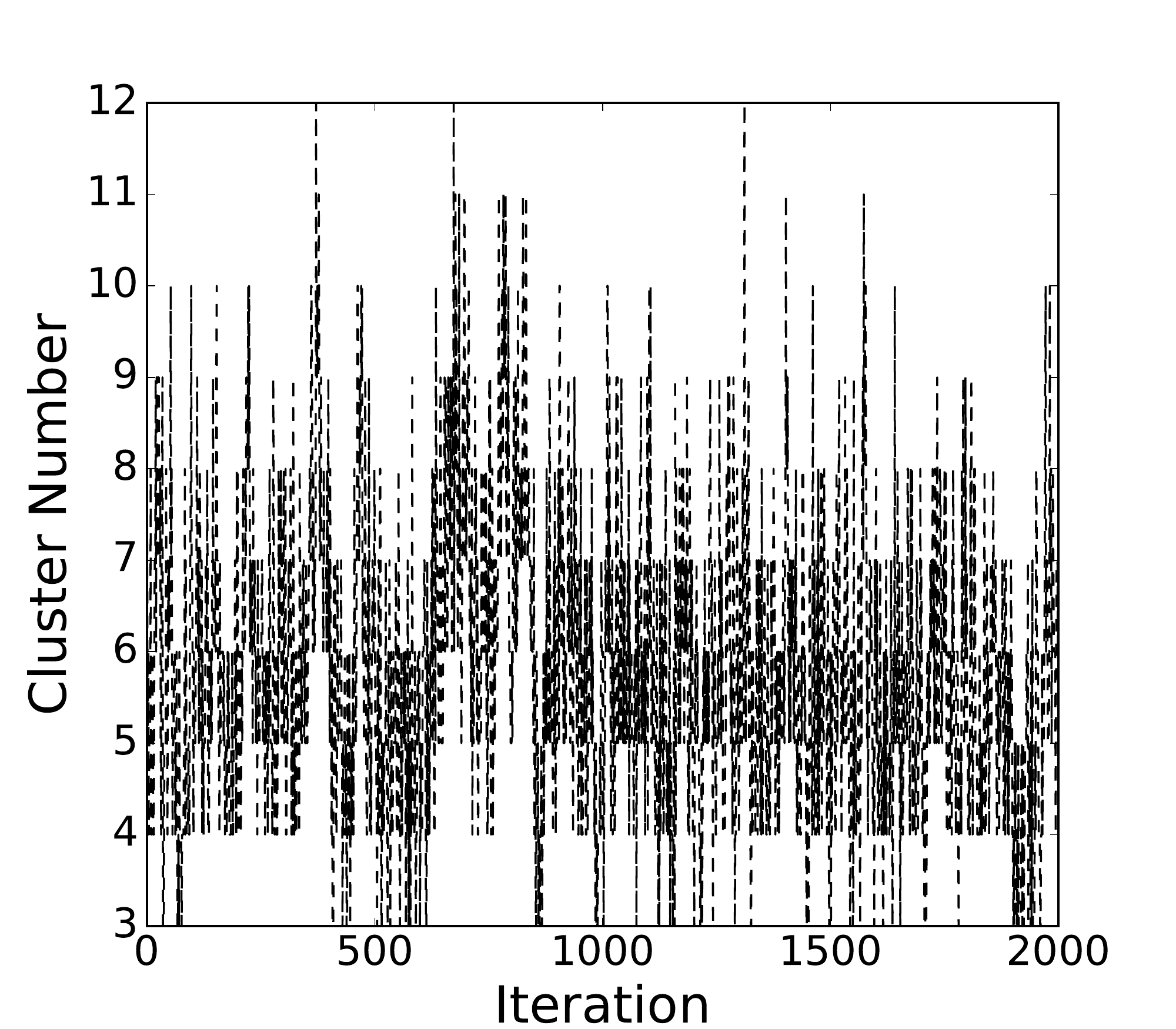} \label{fig:pcrp_clusternum_3methods_traceplot_sim1_crp1}}
	\hfill
	\subfigure[ pCRP]{\includegraphics[width=0.23\textwidth]{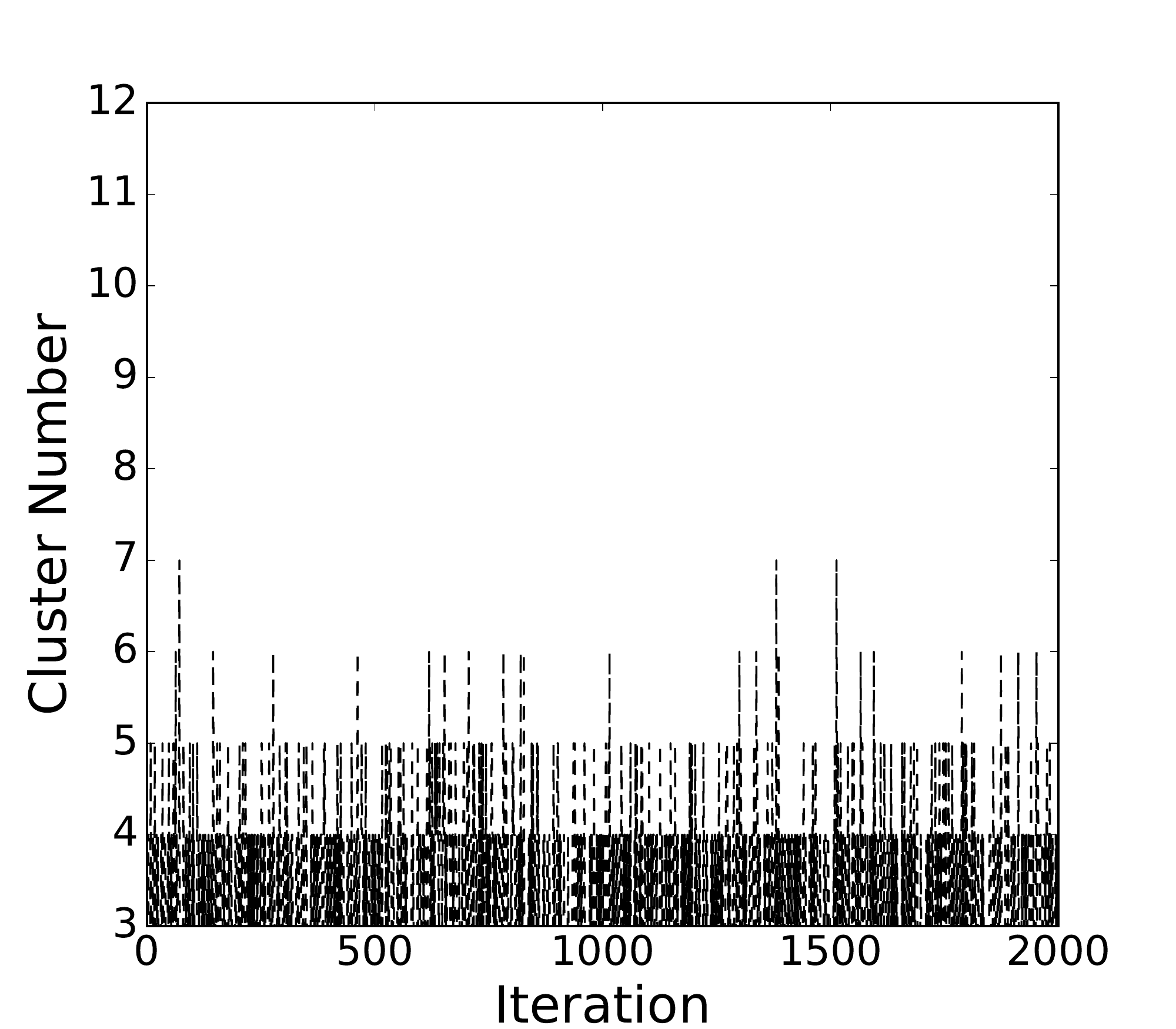} \label{fig:pcrp_clusternum_3methods_traceplot_sim1_pcrp1}}
	\center
	\caption{Traceplots of cluster numbers using the three methods in Sim 1 when $N=2000$. The $x$-axis is the sampling iteration, the $y$-axis is the number of clusters.}
	\label{fig:pcrp_clusternum_3methods_traceplot_sim1}
\end{figure}

%%%%%%% Posterior densities for 3 methods in Sim 1
\begin{figure*}[!h]
	\center
	\subfigure[CRP-Oracle]{\includegraphics[width=0.32\textwidth]{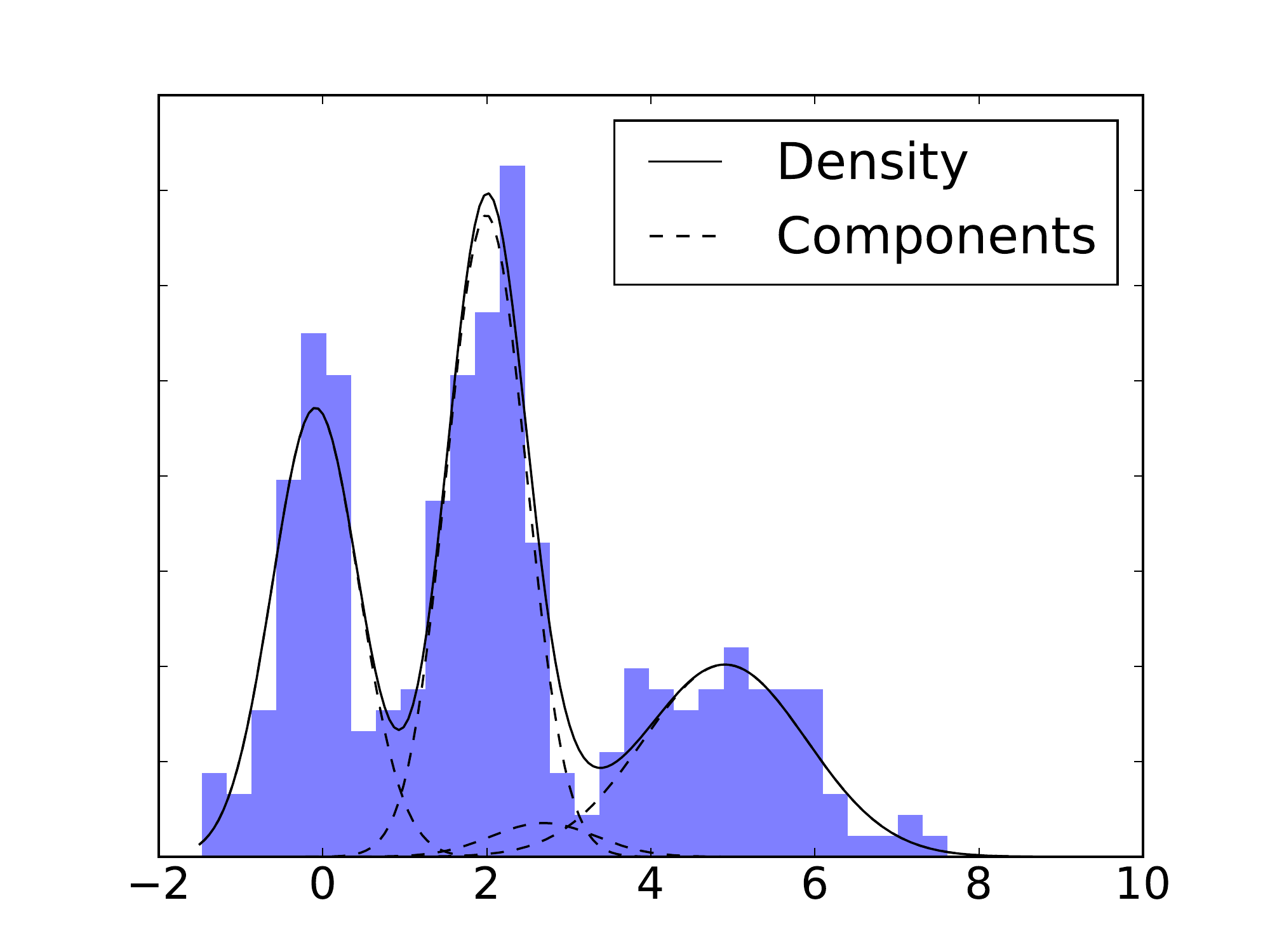} \label{fig:adhoc_3methods_traceplot_sim1_posterior_densities_crp_match}}
	\hfill
	\subfigure[CRP]{\includegraphics[width=0.32\textwidth]{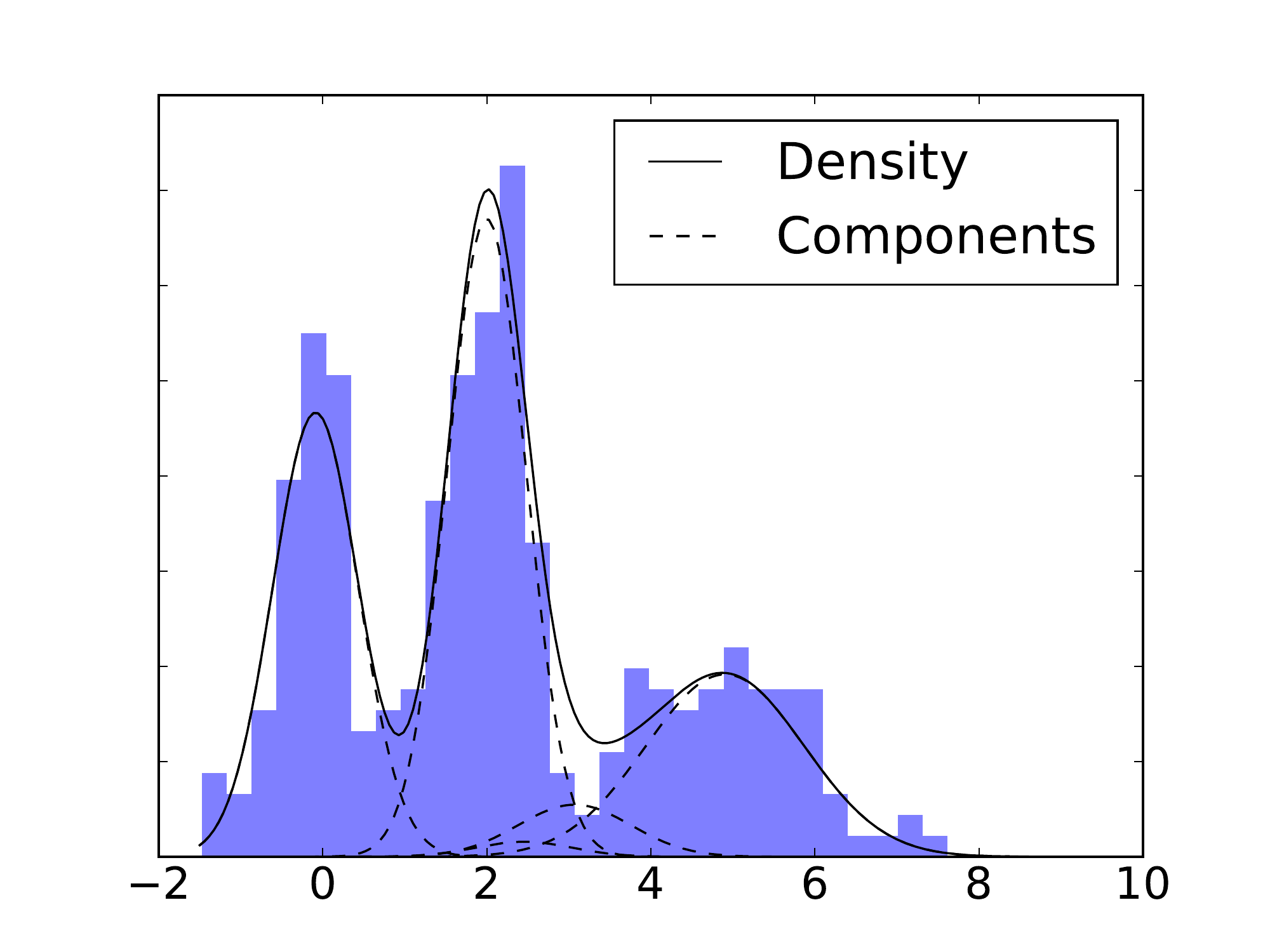} \label{fig:adhoc_3methods_traceplot_sim1_posterior_densities_crp1}}
	\hfill
	\subfigure[pCRP]{\includegraphics[width=0.32\textwidth]{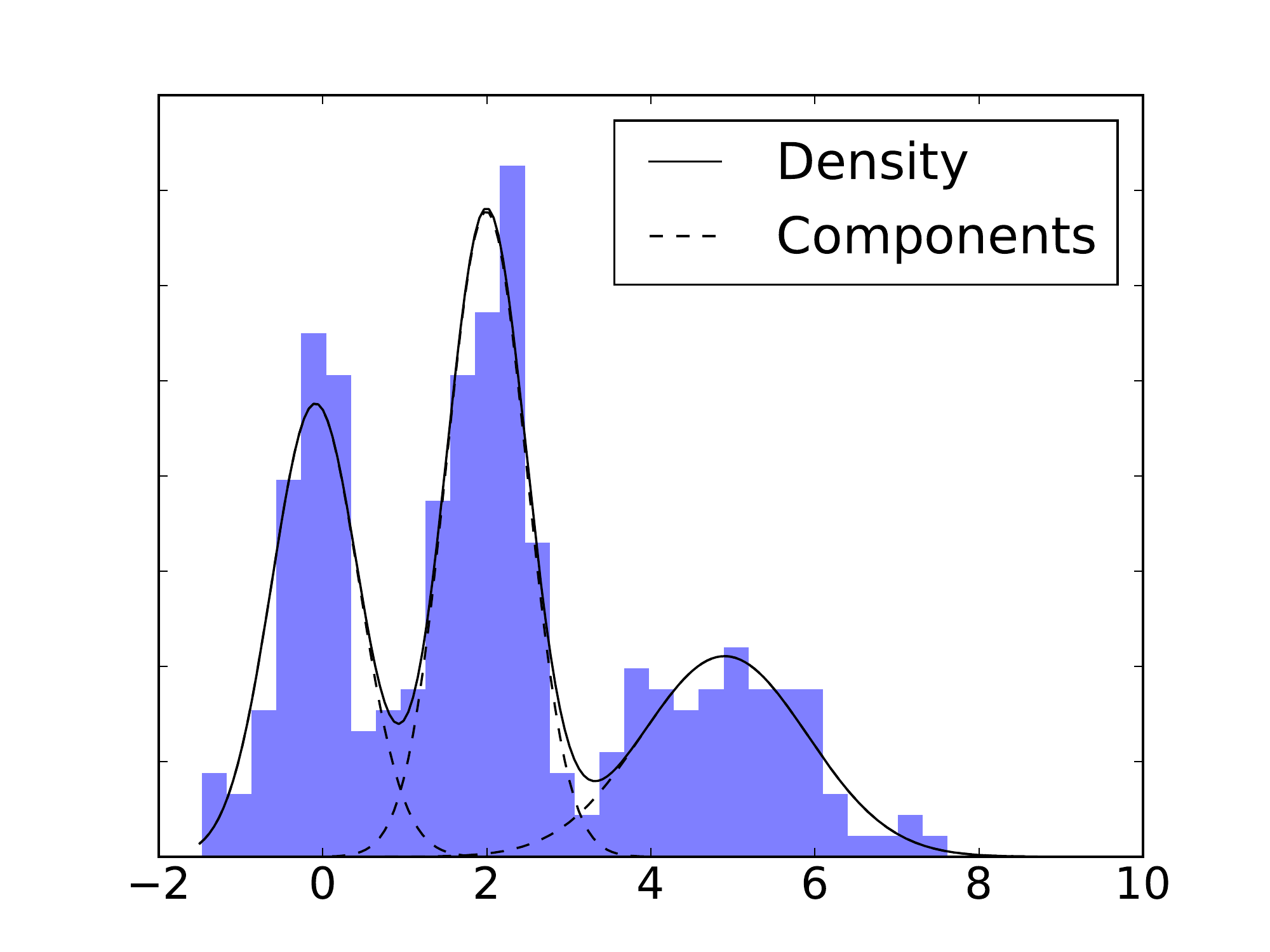} \label{fig:adhoc_3methods_traceplot_sim1_posterior_densities_pcrp1}}
	\center
	\caption{Posterior densities for three methods in Sim 1 when $N=2000$. The dashed lines are weighted components. 
		%We use posterior means to estimate all parameters except for the number of clusters where the posterior mode is used for better visualization.
	}
	\label{fig:pcrp_3methods_traceplot_sim1_posterior_densities}
\end{figure*}

Figure~\ref{fig:pcrp_clusternum_3methods_traceplot_sim1} shows traceplots of posterior samples for the number of clusters for each of the methods in Sim 1.  Clearly pCRP places relatively high posterior probability on three clusters, which is the ground truth. In contrast, CRP has higher posterior variance, systematic over-estimation of the number of clusters, and worse computational efficiency.  The CRP-Oracle has better performance, but does clearly worse than p-CRP, and there is still a tendency for over-estimation.  This demonstrates that one cannot simply fix up or calibrate the CRP by choosing the precision to be appropriately small.
Figure~\ref{fig:posterior_num_clusters} suggests that CRP will have larger probability on larger cluster numbers especially when the sample size increases, while pCRP tends to have larger probability on the true cluster number as the sample size increases. 
For example, in Sim 1, the probability of selecting three clusters increases from 0.55 to 0.68 in pCRP when $N$ increases from 300 to 2000 and the probability for all the other cluster number decreases. 
However, the probability of finding four clusters stabilizes around 0.37 and 0.38 in CRP-Oracle when $N$ increases from 300 to 2000. CRP has increased probability of selecting larger number of clusters (say 5, 6, 7, 8 clusters) when $N$ increases from 300 to 2000. In fact, the proposed pCRP has the largest concentration probability on the true number of clusters among all the three methods including CRP-Oracle, and this observation is consistent between $N = 300$ and $N = 2000$.  

Table~\ref{table:pcrp_simulation_posterior_summary} provides numerical summaries of this simulation. We can see all three methods lead to similar NMI, but pCRP consistently gives the highest value. Furthermore, pCRP leads to the lowest value of VI in most tests. The parsimonious effect of pCRP discussed above is further confirmed by the average and maximum number of clusters; see the columns $K$ and $K_{\max}$ in the table.

The posterior density plots in Figure~\ref{fig:pcrp_3methods_traceplot_sim1_posterior_densities} show that there is one small unnecessary cluster in CRP-Oracle and two small unnecessary clusters in CRP, while all three methods capture the general shape of the true density and thus provide good fitting performance. The over-clustering effect of CRP is much reduced by pCRP as seen in Figure~\ref{fig:adhoc_3methods_traceplot_sim1_posterior_densities_pcrp1}.

%% Note: image name for sim 2 is Sim 1 in the paper. 
%% image name for sim 62 is Sim 2 in the paper
\begin{figure}[h!]
	\center
	\subfigure[Sim 1, $N$=300]{\includegraphics[width=0.22\textwidth]{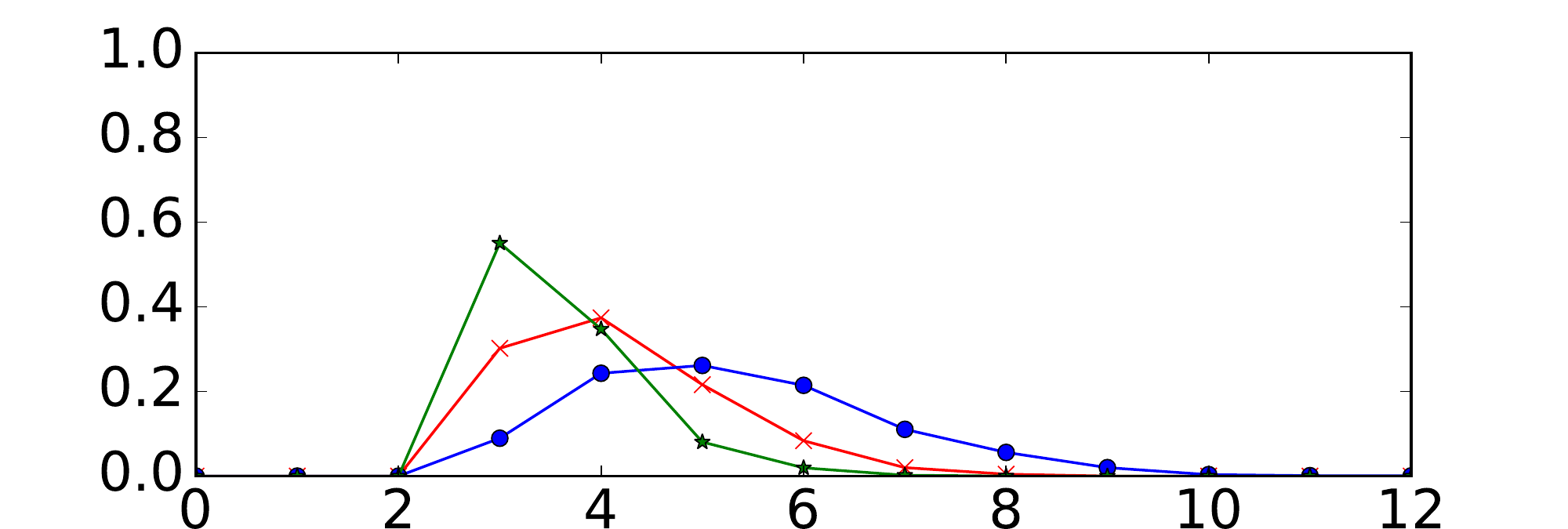} \label{fig:posterior_num_sim2_n300}}
	~
	\subfigure[Sim 2, $N$=300]{\includegraphics[width=0.22\textwidth]{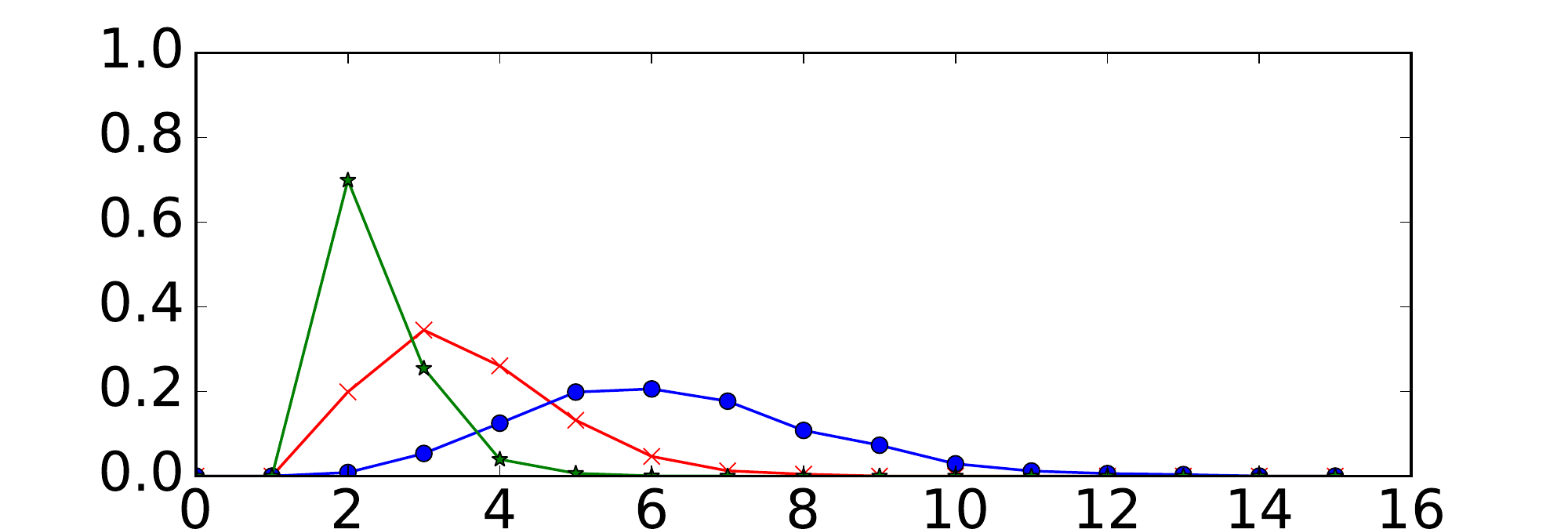} \label{fig:posterior_num_sim6_n300}}
	~
	\subfigure[Sim 1, $N$=2000]{\includegraphics[width=0.22\textwidth]{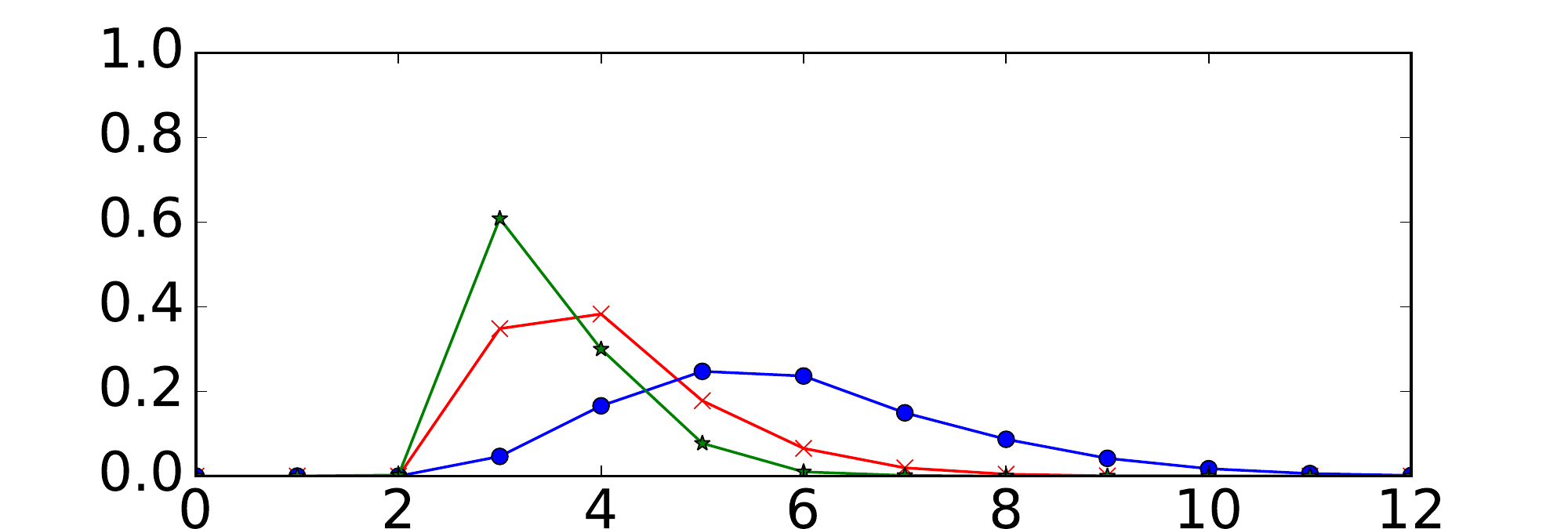} \label{fig:posterior_num_sim2_n2000}}
	~
	\subfigure[Sim 2, $N$=2000]{\includegraphics[width=0.22\textwidth]{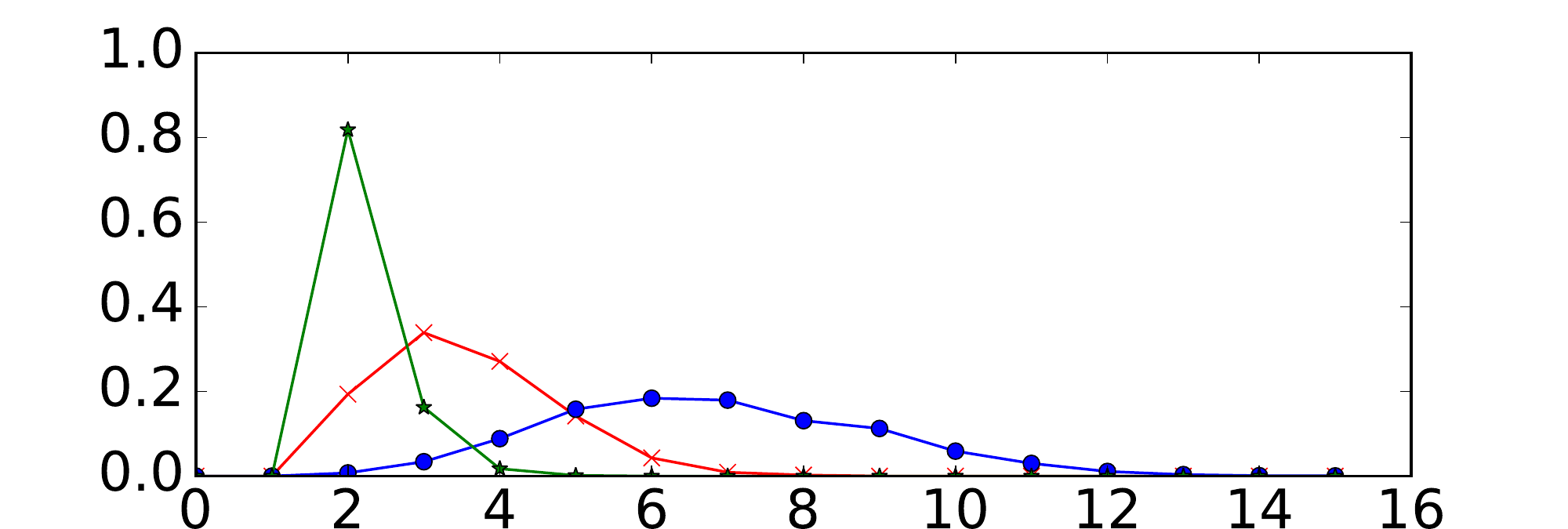} \label{fig:posterior_num_sim6_n2000}}
	\center
	\caption{Estimated posterior of the number of clusters in observed data for CRP-Oracle (red x), CRP (blue circle) and pCRP (green star). }
	\label{fig:posterior_num_clusters}
\end{figure}

\begin{table*}[h!]
\centering
\begin{tabular}{|ccccc|}
\hline
\multicolumn{5}{|c|}{$N=300$}                                                                     \\ \hline
Method               & NMI (SE) & VI (SE) & $K$ (SE) & $K_{\max}$ \\ \hline
Ground truth (Sim 1) & 1.0                & 0.0               & 3                  & -            \\
CRP-Oracle (Sim 1)   & 0.800 (1.1$\times 10^{-3}$)        & 0.669 (4.5$\times 10^{-3}$)       & 4.2 (2.3$\times 10^{-2}$)          & 8            \\
CRP (Sim 1)          & 0.773 (1.2$\times 10^{-3}$)        & 0.795 (5.4$\times 10^{-3}$)       & 5.3 (3.3$\times 10^{-2}$)          & 12           \\
pCRP (Sim 1)         & \textbf{0.827} (0.7$\times 10^{-3}$)       & \textbf{0.580} (4.4$\times 10^{-3}$)       & \textbf{3.6} (1.7$\times 10^{-2}$)          & \textbf{7}            \\ \hline
Ground truth (Sim 2) & 1.0                & 0.0               & 2                  & -            \\
CRP-Oracle (Sim 2)   & 0.937 (0.8$\times 10^{-3}$)        & 0.210 (3.0$\times 10^{-3}$)       & 4.0 (2.3$\times 10^{-2}$)          & 9            \\
CRP (Sim 2)          & 0.917 (0.9$\times 10^{-3}$)        & 0.287 (3.7$\times 10^{-3}$)       & 4.8 (2.9$\times 10^{-2}$)          & 12           \\
pCRP (Sim 2)         & \textbf{0.963} (0.3$\times 10^{-3}$)        & \textbf{0.116} (0.8$\times 10^{-3}$)       & \textbf{3.2} (0.9$\times 10^{-2}$)          & \textbf{6}            \\ \hline \hline
\multicolumn{5}{|c|}{$N=2000$}                                                                    \\ \hline
Method               & NMI (SE) & VI (SE) & $K$ (SE) & $K_{\max}$ \\ \hline
Ground truth (Sim 1) & 1.0                & 0.0               & 3                  & -            \\
CRP-Oracle (Sim 1)   & 0.812 (5.3$\times 10^{-4}$)        & 0.610 (2.6$\times 10^{-3}$)       & 4.0 (2.3$\times 10^{-2}$)          & 10           \\
CRP (Sim 1)          & 0.782 (8.5$\times 10^{-4}$)        & 0.732 (4.0$\times 10^{-3}$)       & 5.8 (3.6$\times 10^{-2}$)          & 12           \\
pCRP (Sim 1)         & \textbf{0.823} (6.6$\times 10^{-4}$)        & \textbf{0.869} (7.3$\times 10^{-3}$)       & \textbf{3.5} (1.6$\times 10^{-2}$)          & \textbf{7}            \\ \hline
Ground truth (Sim 2) & 1.0                & 0.0               & 2                  & -            \\
CRP-Oracle (Sim 2)   & 0.962 (4.8$\times 10^{-4}$)        & 0.122 (1.7$\times 10^{-3}$)       & 4.1 (2.2$\times 10^{-2}$)          & 8            \\
CRP (Sim 2)          & 0.940 (8.2$\times 10^{-4}$)        & 0.205 (3.1$\times 10^{-3}$)       & 5.6 (3.5$\times 10^{-2}$)          & 13           \\
pCRP (Sim 2)         & \textbf{0.977} (1.2$\times 10^{-4}$)        & \textbf{0.072} (0.4$\times 10^{-3}$)      & \textbf{3.1} (0.8$\times 10^{-2}$)          & \textbf{6}            \\ \hline
\end{tabular}
\caption{Comparison of CRP and pCRP on Sim 1 and Sim 2. $K$ is the average number of found clusters. $K_{\max}$ is the maximum number of clusters during sampling. SE is the standard error of mean. Ground truth is calculated using the true assignments.}
\label{table:pcrp_simulation_posterior_summary}
\end{table*}

\begin{figure}[h!]
	\center
	\subfigure[True clustering when $N$=3000]{\includegraphics[width=0.42\textwidth]{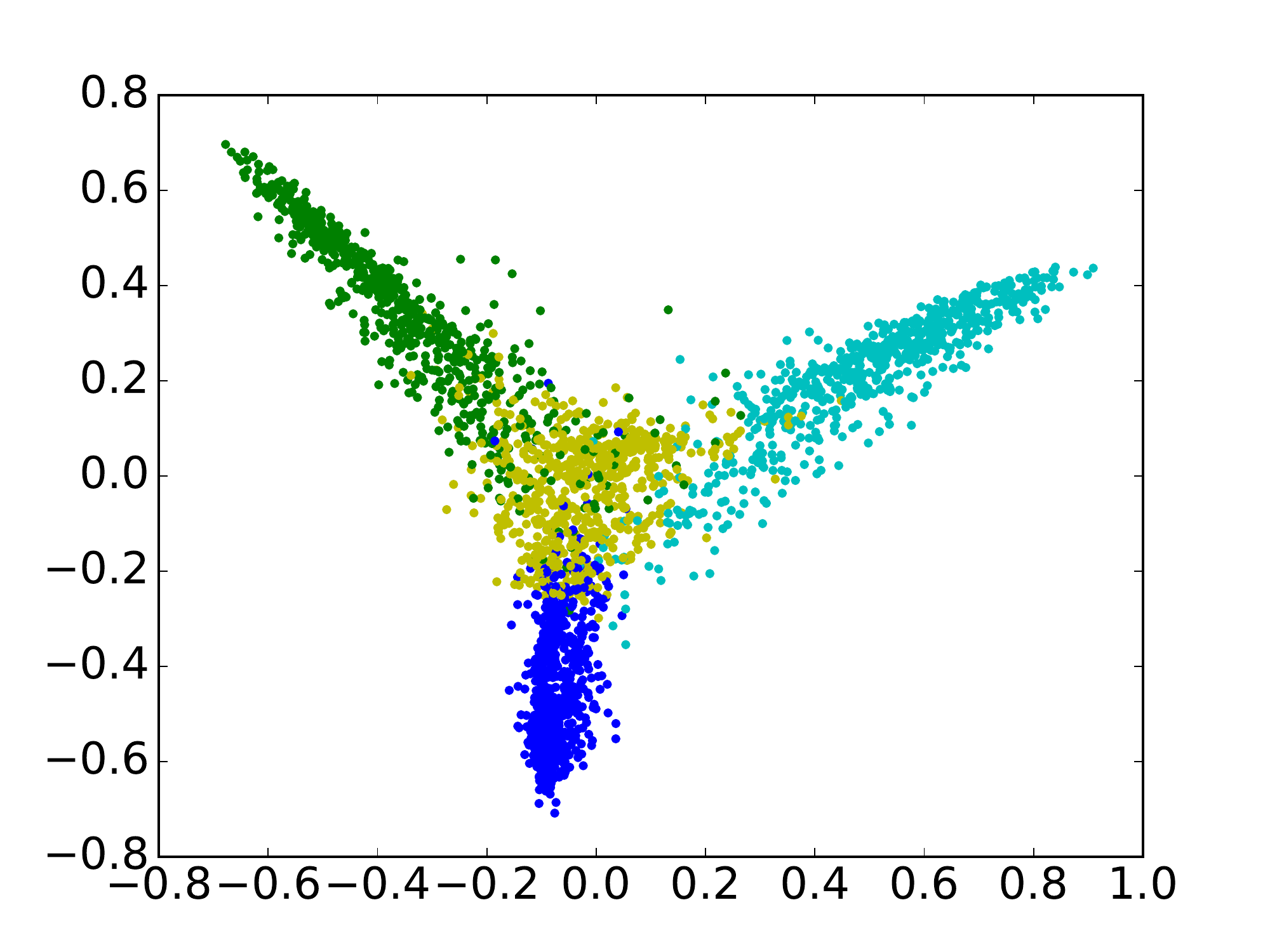} \label{fig:pcrp_digits14_n3000_true_clustering}}
	~
	\subfigure[CRP-Oracle when $N$=3000]{\includegraphics[width=0.228\textwidth]{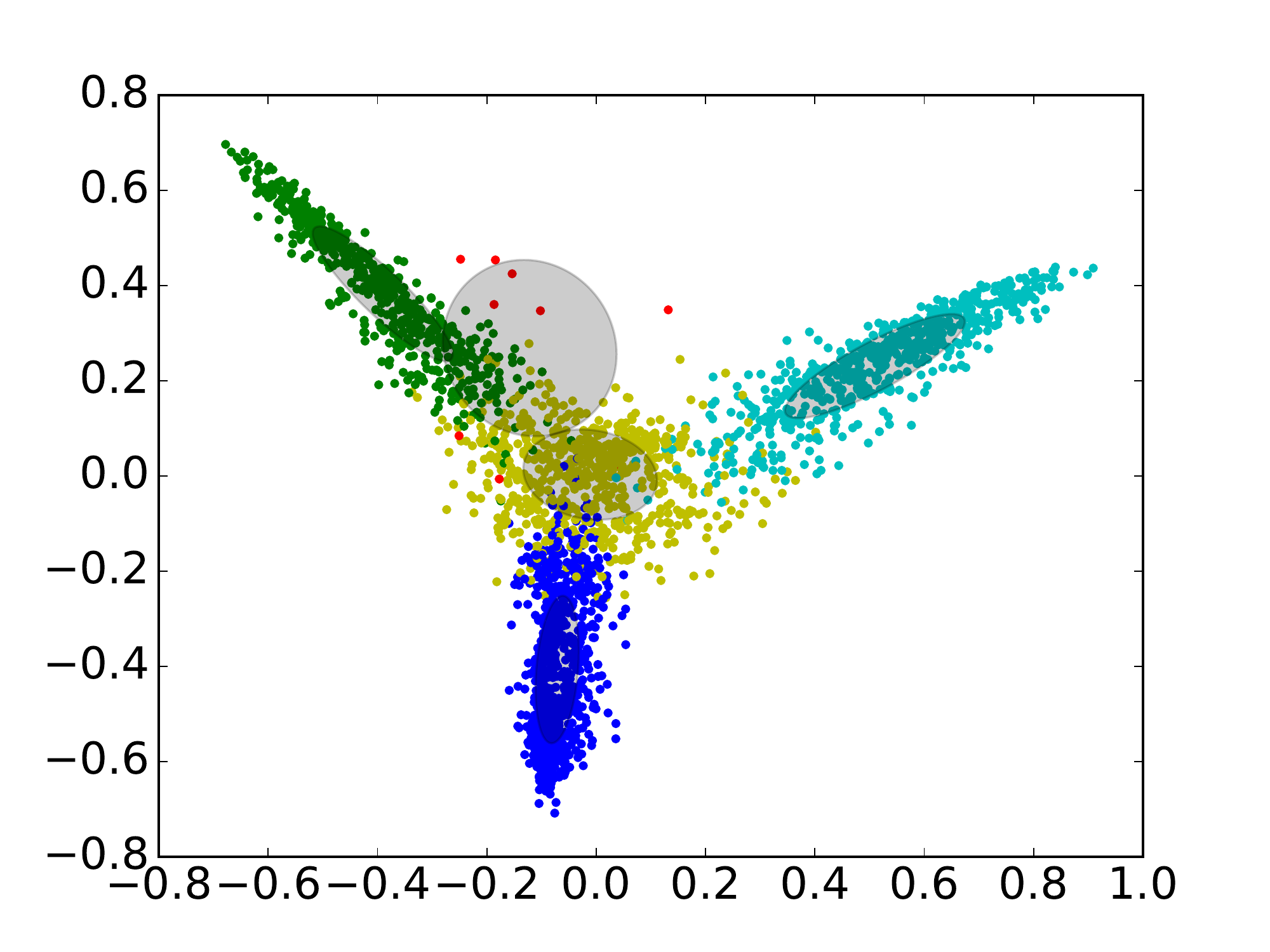} \label{fig:crp_digits14_n3000_crp_match}}
	%~
	%\subfigure[CRP when $N=3000$]{\includegraphics[width=0.3\textwidth]{img_pcrp/digits_posterior_clustering/ellipse_n3000_crp1.pdf} \label{fig:crp_digits14_n3000_crp1}}
	~
	\subfigure[pCRP when $N$=3000]{\includegraphics[width=0.228\textwidth]{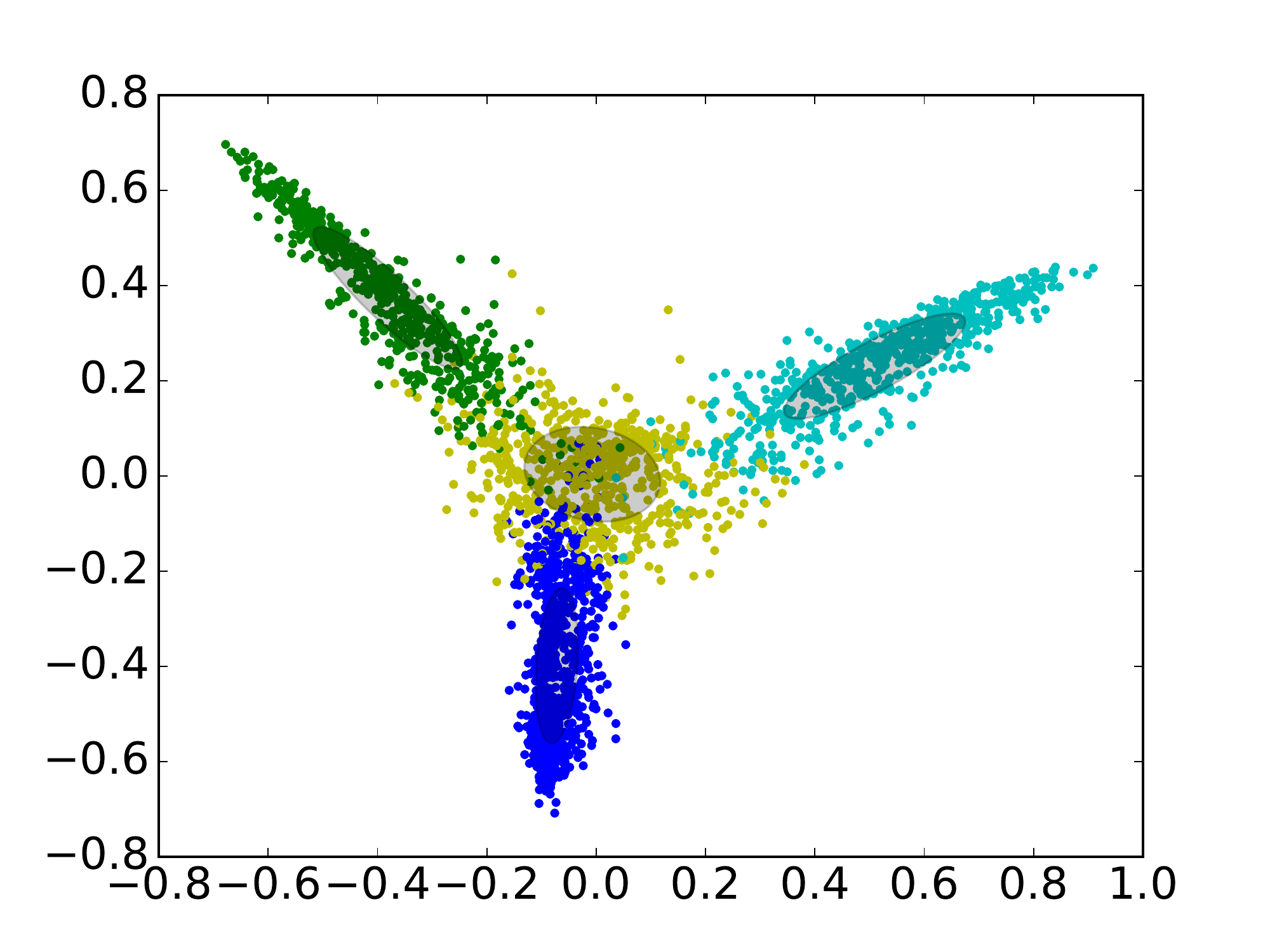} \label{fig:crp_digits14_n3000_pcrp1}}
	\center
	\caption{Results of clustering 3000 randomly sampled digits from 1 to 4 in spectral space. Observations in the same color represent the same digit. CRP-Oracle seems to over-fit the noise (the red cluster). We omit the result of CRP as it is similar to CRP-Oracle.}
	\label{fig:pcrp_digits14}
\end{figure}

\subsection{Digits 1-4}
In this experiment, we cluster 1000 and 3000 digits of the classes 1 to 4 in MNIST data set \cite{lecun2010mnist}, where the four clusters are approximately equally distributed. From cross validation on a different set of 1000 samples, we obtain the power value $r=1.05$. The concentration parameter $\alpha$ in CRP-Oracle is calculated as 0.58 ($N = 1000$) and 0.5 $(N = 3000)$. 

Figure \ref{fig:pcrp_digits14} shows the clustering result of all the three methods for $N=3000$. Both CRP and CRP-Oracle seem to over-fit the data by introducing a small cluster (in red), while pCRP gives a cleaner clustering result with four clusters. This comparison is further confirmed by Table~\ref{table:pcrp_digits_result}, where the average posterior cluster number in CRP apparently increases when $N$ grows to 3000. In contrast, pCRP is closer to the true situation by reducing the over-clustering effect, even compared to CRP-Oracle; see the columns of $K$ and $K_{\max}$. All methods lead to similar NMI but pCRP gives lower VI.

\begin{table*}[h!]
\centering
\begin{tabular}{|ccccc|}
\hline
\multicolumn{5}{|c|}{$N=1000$}                                                                    \\ \hline
Method               & NMI (SE) & VI (SE) & $K$ (SE) & $K_{\max}$   \\ \hline 
Ground truth         & 1.0                & 0                 & 4                  & -            \\
CRP-Oracle           & \textbf{0.651} (3.3$\times 10^{-4}$)        & \textbf{1.382} (1.4$\times 10^{-3}$)       & 4.37 (1.3$\times 10^{-2}$)         & 7            \\
CRP                  & \textbf{0.651} (3.3$\times 10^{-4}$)        & 1.386 (1.4$\times 10^{-3}$)       & 4.58 (1.6$\times 10^{-2}$)         & 8            \\
pCRP                 & \textbf{0.651} (3.3$\times 10^{-4}$)        & \textbf{1.382} (1.4$\times 10^{-3}$)       & \textbf{4.08} (0.6$\times 10^{-2}$)        & \textbf{6}            \\ \hline \hline
\multicolumn{5}{|c|}{$N=3000$}                                                                    \\ \hline
Method               & NMI (SE) & VI (SE) & $K$ (SE) & $K_{\max}$   \\ \hline
Ground truth         & 1.0                & 0.0               & 4                  & -            \\
CRP-Oracle           & 0.651 (2.0$\times 10^{-4}$)        & 1.400 (1.1$\times 10^{-3}$)       & 5.17 (1.2$\times 10^{-2}$)         & 8            \\
CRP                  & 0.651 (2.0$\times 10^{-4}$)        & 1.402 (1.1$\times 10^{-3}$)       & 5.44 (1.6$\times 10^{-2}$)         & 9            \\
pCRP                 & \textbf{0.652} (1.9$\times 10^{-4}$)        & \textbf{1.389} (1.1$\times 10^{-3}$)       & \textbf{4.57} (1.2$\times 10^{-2}$)         & \textbf{7}            \\ \hline
\end{tabular}
\caption{Comparison of CRP-Oracle, CRP and pCRP on a 4 digits subset of MNIST. $K$ is the average number of discovered clusters. $K_{\max}$ is the maximum number of clusters during sampling.}
\label{table:pcrp_digits_result}
\end{table*}

\begin{figure}[!h]
	\center
	\subfigure[The distribution of Old Faithful Geyser after standardization]{\includegraphics[width=0.32\textwidth]{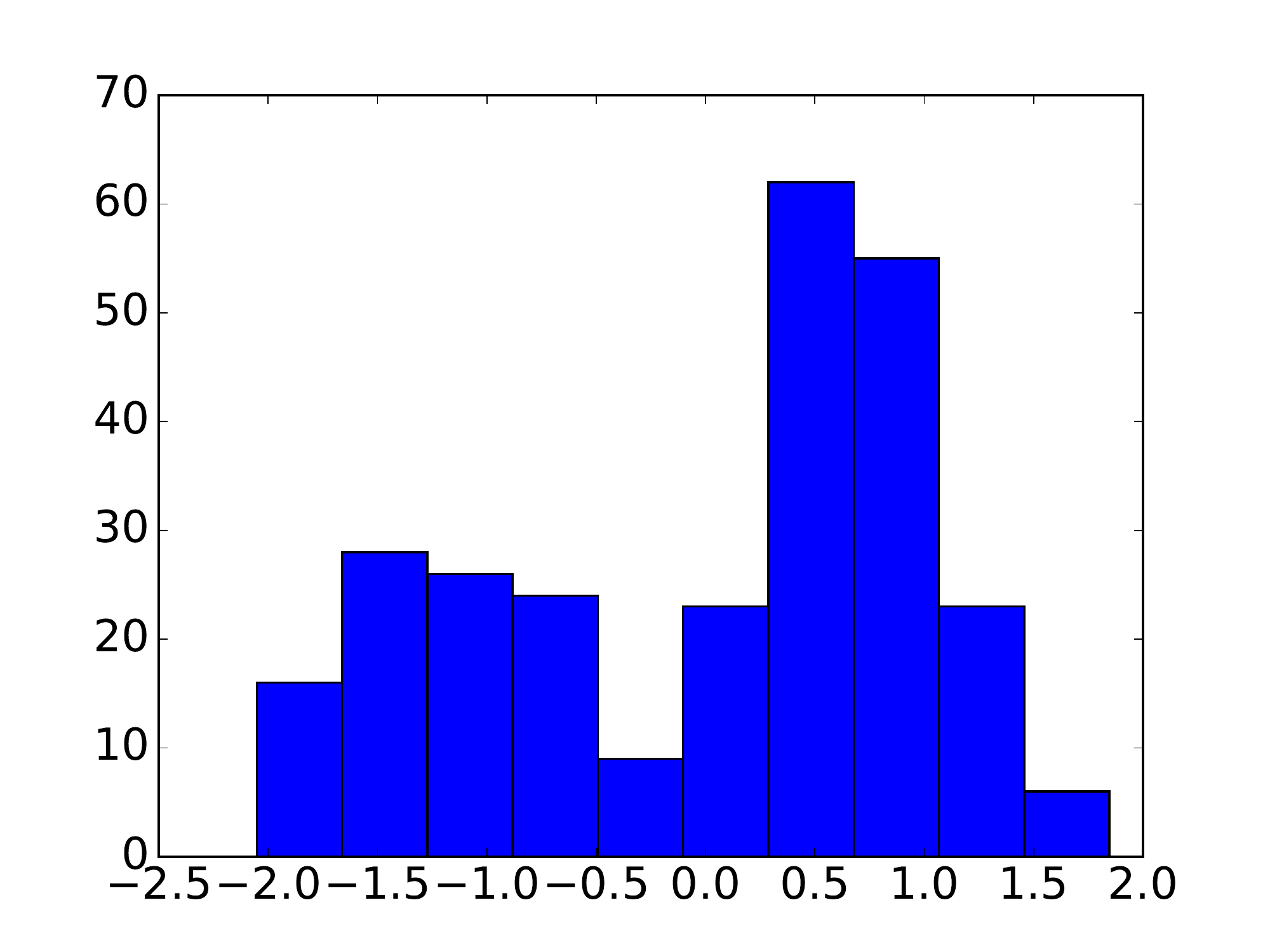} \label{fig:adhoc_oldfaithful_data_distribution}}
	~
	\subfigure[Clustering results using CRP-Oracle, CRP, pCRP and manual clustering.]{\includegraphics[width=0.32\textwidth]{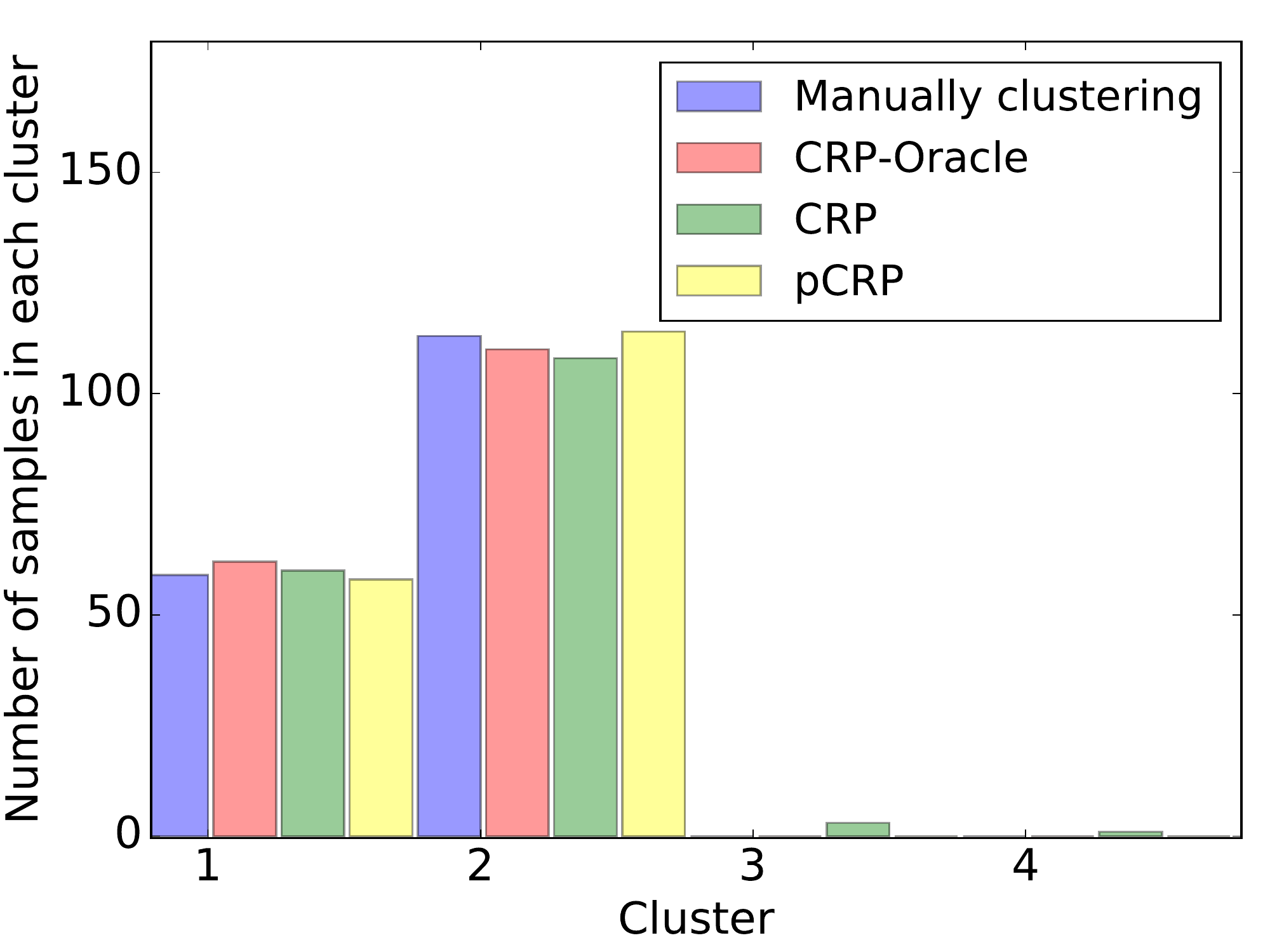} \label{fig:adhoc_oldfaithful_bar_comparison}}
	\center
	\caption{Clustering result for Old Faithful Geyser}
	\label{fig:adhoc_oldfaithful}
\end{figure}

\subsection{Old Faithful Geyser}
The Old Faithful Geyser data ($N=272$) are widely used to illustrate the performance of clustering algorithms. 
We use a test sample of 100 in CV leading to the power value $r=1.11$. We compare all methods on the other 172 data points. A manual clustering that consists of two Gaussian components is viewed as the ground truth. The concentration parameter is 0.39 in CRP-Oracle.  Figure~\ref{fig:adhoc_oldfaithful_bar_comparison} shows the size of each component obtained from all methods and the manual clustering. We can see that there are two mixture components in CRP-Oracle and pCRP, and four mixture components in the CRP method. In this case where the sample size is relatively small, we again see that pCRP successfully suppresses small components and generates more parsimonious results than CRP.

%Table \ref{table:ifmm_old_faithful_4_cluster_num} shows the comparison of traceplot for the cluster number for several methods. 

%\begin{table}[]
%\centering
%\caption{Comparison of traceplot for the cluster number on Old Faithful Geyser}
%\label{my-label}
%\begin{tabular}{|c|l|}
%\hline
%Old Faithful Geyser & \begin{tabular}[c]{@{}l@{}}Traceplot for the number of clusters\\  (number of clusters : iterations during %sampling)\end{tabular} \\ \hline
%Gibbs               & 3: 146, 4: 144, 5: 77, 2: 69, 6: 41, 7: 18, 8: 3, 9: 1, 10: 1                                                                     \\ \hline
%pCRP                & 2: 234, 3: 182, 4: 75, 5: 7, 6: 2                                                                                                 \\ \hline
%Ada-pCRP        & 3: 187, 2: 177, 4: 110, 5: 20, 6: 6																 \\ \hline
%\end{tabular}
%\label{table:ifmm_old_faithful_4_cluster_num}
%\end{table}

%\begin{itemize}
%\item  This is the first entry in our list
%\end{itemize}

\pagebreak
%\cleardoublepage
%\clearpage

\bibliography{bibliography}
\bibliographystyle{icml2018}

\end{document}